\documentclass[sigconf]{acmart}
\AtBeginDocument{%
  }

\copyrightyear{2026}
\acmYear{2026}
\setcopyright{cc}
\setcctype{by}
\acmConference[HCOMP 2026]{2026 ACM Conference on Human-AI Complementarity and Alignment}{September 27--30, 2026}{Alexandria, VA, USA}
\acmBooktitle{2026 ACM Conference on Human-AI Complementarity and Alignment (HCOMP 2026), September 27--30, 2026, Alexandria, VA, USA}
\acmDOI{10.1145/3834580.3838745}
\acmISBN{979-8-4007-2894-5/2026/09}

\acmSubmissionID{77}

\usepackage{subcaption}
\usepackage{colortbl}
\usepackage{mathtools}
\usepackage{makecell}
\usepackage{enumitem}
\usepackage{amsthm}
\usepackage{bm}
\usepackage{multirow}
\usepackage{cleveref}
\crefformat{footnote}{#2\footnotemark[#1]#3}

\usepackage{listings}
\DeclareCaptionStyle{ruled}{labelfont=normalfont,labelsep=colon,strut=off} 
\floatstyle{ruled}
\newfloat{listing}{tb}{lst}{}
\floatname{listing}{Listing}

\usepackage{fancybox}

\newenvironment{promptbox}[2]{%
  \begin{Sbox}%
  \begin{minipage}{\dimexpr#1-2\fboxsep-2\fboxrule\relax}%
    \setlength{\parindent}{0pt}%
    \vspace*{-\fboxsep}%
    \hspace*{-\fboxsep}%
    \makebox[0pt][l]{%
      \colorbox{promptboxtitlebg}{%
        \makebox[\linewidth][l]{\bfseries\footnotesize\strut #2}}}%
    \par\smallskip
}{%
  \end{minipage}%
  \end{Sbox}%
  \fbox{\TheSbox}%
}

\definecolor{promptboxtitlebg}{gray}{0.8}

\definecolor{dkgreen}{rgb}{0,0.6,0}
\definecolor{gray}{rgb}{0.5,0.5,0.5}
\definecolor{mauve}{rgb}{0.58,0,0.82}

\lstnewenvironment{pythontable}
  {\lstset{language=Python, frame=none,
  aboveskip=-1.5mm,
  belowskip=-2mm,
  columns=fullflexible,  
  basicstyle=\tiny\ttfamily,
    }%
  }
  {}

\newtheorem{definition}{Definition}
\newtheorem{theorem}{Theorem}

\newtheorem{lemma}{Lemma}
\newtheorem{claim}{Claim}

\makeatletter
\newcommand{\thickhline}{%
    \noalign {\ifnum 0=`}\fi \hrule height 1pt
    \futurelet \reserved@a \@xhline
}
\newcolumntype{"}{@{\hskip\tabcolsep\vrule width 1pt\hskip\tabcolsep}}
\makeatother

\makeatletter
\newcommand{\labelname}[1]{
  \def\@currentlabelname{#1}}%
\makeatother

\usepackage{microtype}
\newcommand{\doc}[1]{{\color{teal}{#1}}}

\usepackage{algorithm}
\usepackage[noend]{algorithmic}

\DeclareMathOperator*{\argmax}{argmax}

\DeclareCaptionFormat{cont}{#1 (cont.)#2\footnotesize#3\par}

\begin{document}

\title{Guiding Worker Self-Selection in Crowdsourcing Contests: An LLM-Augmented Algorithmic Approach}
\titlenote{Part of our work originally appeared in an extended abstract from AAMAS 2020 \citep{chan2020price}.}

\author{Nguyen Thach}
\correspondingauthor
\email{nate.thach@huskers.unl.edu}
\author{Hau Chan}
\email{hchan3@unl.edu}
\affiliation{%
  \institution{University of Nebraska-Lincoln}
  \city{Lincoln}
  \state{Nebraska}
  \country{USA}
}

\author{David Parkes}
\email{parkes@eecs.harvard.edu}
\author{Karim Lakhani}
\email{klakhani@g.harvard.edu}
\affiliation{%
  \institution{Harvard University}
  \city{Cambridge}
  \state{Massachusetts}
  \country{USA}
}





\renewcommand{\shortauthors}{Thach et al.}

\begin{abstract}
Crowdsourcing platforms coordinate large pools of online workers who compete in contests to produce solutions for clients, from logo design to machine-learning challenges. A defining feature of these platforms is that workers \emph{self-select}: Each chooses which contests to enter and how much effort to invest, and these choices are strategic rather than aligned with what the platform needs. As a result, the platform may end up with too few participants, or too little effort, on the contests where high-quality solutions matter most---while individual workers may end up dissatisfied, regretting choices that left them worse off than an alternative they could have taken. To reconcile these competing interests, platforms increasingly offer workers recommendations on where to compete. We study how to generate such recommendations as the game of \emph{self-selection in Tullock contests} (SSTC), a two-stage model in which workers first choose contests and then compete within them. We introduce GRAF, a greedy algorithmic framework that produces self-selection outcomes in polynomial time by ordering workers according to a score vector, with provable guarantees---zero worker regret and platform optimality---for special cases of SSTC. Because a good ordering is difficult to design by hand once workers are heterogeneous, we propose LLMScore, an LLM-driven evolutionary framework that automatically \emph{designs} the scoring algorithm GRAF uses. LLMScore addresses two obstacles to applying existing LLM-based methods here: jointly optimizing solution quality and worker satisfaction, and the intractability of evaluating worker regret. It is trained only on small instances of a single setting yet transfers to larger and structurally different settings, and---because its output is human-readable code---a platform operator can inspect and adjust the recommendation logic it learns. Across 1,000 synthetic instances spanning four settings, GRAF augmented by LLMScore consistently yields high-quality or even near-optimal outcomes with low worker regret, benefiting platform and workers alike.
\end{abstract}




\maketitle

\section{Introduction}

Crowdsourcing platforms such as \href{https://www.topcoder.com/}{Topcoder} and \href{https://www.freelancer.com/}{Freelancer} have become a central mechanism for completing work and solving problems at scale \citep{king2013using,lakhani2013prize,cricelli2022crowdsourcing}, letting clients obtain cost-effective solutions without hiring additional staff.
Such platforms organize a client's problem as an online \emph{contest}, where online workers, or \emph{contestants}, sign up to compete by submitting solutions. After the contest, a reward goes to the winner and the submitted solutions are returned to the client.
Contest problems range from simple logo and website designs to programming-intensive algorithmic and machine learning challenges.

What makes coordinating such a platform hard is that the workers are not assigned to tasks---they \emph{self-select}.
Each worker decides, on their own, which contest to enter and how much effort to invest, and may do so strategically to maximize their own payoff \citep{dipalantino2009crowdsourcing}.
The platform, by contrast, wants every contest to attract enough participants and effort to yield high-quality solutions.
These two interests pull apart: The distribution of effort that emerges from workers acting in their own interest can differ sharply from the distribution the platform would prefer, so the quality of solutions a client receives depends crucially on workers' strategic behavior.
Worse, that behavior may not even settle into a stable arrangement \citep{bernergard2017self}---an \emph{equilibrium} in which no worker wishes they had chosen differently.

Two quantities capture this tension in human terms.
From the platform's side, the \emph{overall effort} across contests is a natural proxy for the quantity and quality of solutions produced, since even non-winning submissions have value (e.g., multiple classifiers can be combined into a stronger ensemble) \citep{lakhani2013prize}.
From the worker's side, \emph{regret} measures dissatisfaction: how much utility a worker loses by entering the contest they did, relative to the best contest they could have entered given everyone else's choices.
A worker with high regret is one who, after the fact, wishes the platform had steered them elsewhere---and such workers are precisely those a platform risks losing.
Because workers are not under contract, the platform's only lever is advice: recommendations on which contest(s) to join \citep{difallah2013pick,king2013using}.
The goal is recommendations that are simultaneously near-optimal for the platform and low-regret for workers, so that both parties benefit.

We formalize this as {\em self-selection in Tullock contests} (SSTC): a two-stage game in which workers first choose which contests to enter, then compete within each contest, with the probability of winning increasing in one's own effort relative to others'.
The Tullock model is the standard formalism for such effort-based competition.
This leads to the question we address.

\begin{quote}
 \emph{Can we design algorithms for general SSTC that generate near-optimal self-selection outcomes with low dissatisfaction from workers?}
\end{quote}

\paragraph{Existing Efforts and Their Limits.}
Prior work \citep{bernergard2017self,cohen2026algorithmic} studies special cases of SSTC and offers partial equilibrium characterizations of self-selection outcomes.
However, these results rest on restrictive assumptions---workers share identical valuations, identical production capabilities and effort costs, and may select only a single contest---that strip out the heterogeneity of real worker populations and offer no practical, deployable recommendation procedure.

\paragraph{Our Approach.}
We first propose GRAF, a greedy algorithmic framework that orders workers according to a score vector $\sigma$ and then iteratively assigns each to a contest and determines their effort, conditioned on the choices already made.
The difficulty is the ordering: When workers differ in their valuations, production efficiency, and effort costs, choosing a good score vector by hand is far from obvious, which motivates a data-driven approach to designing $\sigma$ automatically.

Large language models (LLMs) have shown strong capability in algorithmic design \citep{liu2024systematic} and programming \citep{jiang2024survey}.
In particular, embedding an LLM in an evolutionary search---maintaining a population of candidate heuristics and repeatedly prompting the model to improve on the best ones---has produced state-of-the-art results on combinatorial problems such as traveling salesman and knapsack problems \citep{romera2024mathematical,liu2024evolution,zheng2025monte}.
Because these methods excel precisely at designing priority scores for greedy procedures (e.g., selecting the next node to add in TSP tour construction), they are a promising way to design the ordering $\sigma$ that GRAF needs.

Two obstacles stand in the way of applying them directly.
(Challenge~I) Existing methods target objectives with a single, well-behaved notion of solution quality (e.g., tour length and total weight), which makes their ``fitness function'' easy to define.
Our setting has two objectives---\emph{overall effort} and worker \emph{regret}---whose scales vary enormously across instances\footnote{In general, regret is unbounded.}, so a fitness function must capture both at once.
(Challenge~II) Existing methods assume cheap-to-evaluate objectives, but computing worker regret is intractable \citep{aissi2009min}, especially when no closed-form equilibrium effort is available for the second stage, making naive evaluation prohibitively expensive.

\paragraph{Contributions.}
Our contributions are as follows.

\begin{enumerate}
    \item We propose GRAF, a general algorithmic framework for SSTC that, under a suitable $\sigma$, provably returns a zero-dissatisfaction, or \emph{zero-regret/equilibrium}, outcome (Theorem~\ref{thm:psne}) and even a platform-optimal one (Theorem~\ref{thm:opt}) in polynomial time for special cases of SSTC.

    \item We propose LLMScore, an LLM-driven evolutionary framework that automatically designs the scoring algorithm $\sigma$ that GRAF uses in general SSTC games. Addressing Challenge~I, we devise a simple-yet-effective fitness-penalizing scheme that nudges algorithms toward platform optimality while filtering out those that leave workers with high regret. Addressing Challenge~II, we train LLMScore only on small instances of a single SSTC setting and reuse the learned algorithms on unseen, larger, structurally different settings, eliminating expensive per-setting training.

    \item We conduct rigorous experiments on 1,000 synthetic SSTC instances across four problem settings of increasing complexity to demonstrate the following advantages of LLMScore:
    \begin{itemize}
        \item \emph{Performance}: GRAF augmented by LLMScore consistently yields high-quality or even near-optimal self-selection outcomes with low worker regret across all considered settings, including those it was never trained for.
        \item \emph{Interpretability}: LLMScore produces transparent, human-readable code that a platform operator can inspect, audit, and adjust (see Figures \ref{fig:learned-alg1} and \ref{fig:learned-alg3} for some examples).
        \item \emph{Scalability}: By virtue of its generalizability (from tackling Challenge~II), the per-instance runtime of GRAF augmented by LLMScore is comparable to GRAF with handcrafted $\sigma$ (which runs in polynomial time), since training is amortized offline.
    \end{itemize}
\end{enumerate}

\paragraph{Outline.}
The remainder of our paper is organized as follows.
Section~\ref{sec:prelim} formalizes the Tullock contest model and the SSTC game. Section~\ref{sec:graf} presents GRAF and establishes its theoretical guarantees. Section~\ref{sec:llmscore} introduces LLMScore. Section~\ref{sec:exp} presents experimental results. We discuss related work in Section~\ref{sec:related} and conclude our work in Section~\ref{sec:conclusion}.
The appendices contain figures for our prompts (Figures~\ref{fig:prompt-task}--\ref{fig:evol_prompt-design}), code template (Listing~\ref{lst:code-template}), designed algorithms (Figures~\ref{fig:learned-alg1} and \ref{fig:learned-alg3}), and additional results (Figures~\ref{fig:box-ssstc_general}--\ref{fig:tradeoff-sstc_general}).
Complete proofs and full implementation details are provided in our supplementary material attached at the end of the document.

\section{Preliminaries}\label{sec:prelim}

In this section, we first revisit the model of a single Tullock contest in Section \ref{subsec:tullock}, then formalize the general SSTC model involving multiple Tullock contests in Section \ref{subsec:sstc} and state the platform's objective in Section \ref{subsec:obj}.

\subsection{Tullock Contest}\label{subsec:tullock}
We consider the most general model of Tullock contests \citep{tullock1980efficient}. 
Let $N=\{1,\ldots,n\}$ denote the set of contestants in a contest.
Each contestant $i\in N$ has a value $v_i>0$ for the reward associated with the contest (e.g., monetary prize, skill points, or reputation) and wins the contest with probability that depends on their exerted effort and the exerted efforts of other contestants.
In particular, each $i$ is characterized by a \emph{production function} $f_i$.
At effort $e_i$, contestant's $i$'s production is $f_i(e_i)=q_ie_i^{r_i}$, where $q_i>0$ describes $i$'s \emph{production efficiency}, and $r_i>0$ captures the \emph{elasticity of effort} \citep{tullock1980efficient}.
Let $\mathbf{e}=(e_1,\ldots,e_n)\in\mathbb{R}_0^{+\,n}$ be the joint-effort profile of the contestants where $e_i\ge0$ is the effort of contestant $i\in N$.
The probability of $i$ winning the contest is then $p_i(\mathbf{e}) = \frac{f_i(e_i)}{\sum_{l \in N} f_l(e_l)}$.
If all $e_i$'s are zero, then $p_i(\mathbf{0}) = \frac{1}{n}$ for all $i \in N$ where $\mathbf{0} = (0, \ldots, 0)$.
Given a joint-effort profile $\mathbf{e} = (e_i, \mathbf{e}_{-i})$, the expected utility of $i \in N$ is $u_i(e_i, \mathbf{e}_{-i}) = v_i p_i(\mathbf{e}) - c_i(e_i)$, where $c_i:\mathbb{R}^+_0\rightarrow\mathbb{R}^+_0$ is the \emph{effort cost function} of $i$.

In the special case where $f_i(e_i)=e_i$ and $c_i(e_i)=e_i$, for every $i\in N$, we have Theorem~\ref{thm:tullock}.

\begin{theorem}[\citep{hillman1989politically,stein2002asymmetric}]\label{thm:tullock}
Given a Tullock contest with values $\{v_i\}_{i \in N}$, where $v_1 \geq v_2 \geq \cdots \geq v_n$, there is a unique pure-strategy Nash equilibrium (PSNE) $\mathbf{e}^* \in \mathbb{R}_0^{+\,n}$ where (1) the equilibrium effort level is ordered based on the values: $e_1^* \geq e_2^* \geq \cdots \geq e_n^*$; (2) the number of active players (with nonzero efforts) is the smallest $p$, in the ordering, such that $v_{p+1} < \frac{p-1}{\Gamma_p}$ (with $\Gamma_p = \sum_{i \leq p} \frac{1}{v_i}$); and (3) the equilibrium effort is characterized by the closed form: $e_i^* = \frac{p-1}{\Gamma_p}\left(1 - \frac{p-1}{v_i \Gamma_p}\right)$ for the first $p$ contestants and $e_i^* = 0$ for $i = p+1, \ldots, n$.
\end{theorem}
In such a Tullock contest, from the above PSNE characterization, there are always at least two active contestants for $n\ge2$.

\subsection{Self-Selection in Tullock  contests}\label{subsec:sstc}

The game of self-selection in Tullock contests (SSTC) is played in two stages. First, each contestant selects which contests to participate in, and, subsequently, the contestants compete in each  contest, knowing the contestants they compete against \citep{dipalantino2009crowdsourcing, bernergard2017self,cohen2026algorithmic}.

Formally, let $M = \{1, \ldots, m\}$ be the set of $m$ distinct Tullock contests, and $N = \{1, \ldots, n\}$ be the set of $n$ contestants. A (pure) self-selection strategy of a contestant $i$ is $s_i = (a_i, e_i)$ with a binary selection vector $a_i$ of size $m$ and an effort vector $e_i$ of size $m$, where $a_{ij} = 1$ denotes the selection of contest $j$, and $e_{ij}\geq 0$ denotes $i$'s effort in contest $j \in M$.
The (pure) self-selection strategy set of a contestant $i \in N$ is $S_i \in A \times \mathbb{R}_0^{+\,m}$, with $A = \{a_i \in \{0,1\}^m \mid \sum_{j \in M} a_{ij} = k_i\}$, where $1\le k_i<m$ is the number of contests in which $i$ seeks to participate\footnote{In practice, $k_i$ represents $i$'s demand based on their e.g., past participation and profile data. The platform then provides recommendations respecting this demand \citep{difallah2013pick}.}. 
We let $S = S_1 \times \cdots \times S_n$ 
and $v_{ij}$ be the value of $i \in N$ for winning contest $j\in M$. 

Given a pure-strategy profile or self-selection outcome $\mathbf{s} = (s_i, \mathbf{s}_{-i}) \in S$, the utility of contestant $i$ is

\begin{equation}\label{eq:utility}
    u_i(s_i, \mathbf{s}_{-i}) = \sum_{j=1}^m a_{ij}\cdot\Big(v_{ij} p_{ij}(\mathbf{s}) - c_{ij}(e_{ij})\Big),
\end{equation}
where $c_{ij}:\mathbb{R}^+_0\rightarrow\mathbb{R}^+_0$ is the effort cost function of contestant $i$ in contest $j$ and
\begin{equation}\label{eq:win-prob}
    p_{ij}(\mathbf{s}) = \frac{a_{ij}\cdot f_{ij}(e_{ij})}{\sum_{l\in N}a_{lj}\cdot f_{lj}(e_{lj})}
\end{equation}
is $i$'s probability of winning contest $j$ given $\mathbf{s}$.
We define $f_{ij}(e_{ij})=q_{ij}\cdot e_{ij}^{r_{ij}}$ and $c_{ij}(e_{ij})=c_{ij}\cdot e_{ij}^{d_{ij}}$, where $q_{ij}, r_{ij}, c_{ij}, d_{ij}>0$.
Note that within contest $j$, $p_{ij}(\mathbf{s}) = 0$ when $a_{ij} = 0$ for every $i \in N$, and $p_{ij}(s) = \frac{1}{\sum_{l=1}^n a_{lj}}$ when $e_{ij} = 0$ and $a_{ij} = 1$ for each $i \in N$. When there is a single contestant, the contestant wins the prize with probability 1 with zero effort (i.e., $p_{ij}(\mathbf{s}) = 1$ and $e_{ij} = 0$ if $a_{ij} = 1$ and $\sum_{l \in N} a_{lj} = 1$).

Due to the sequential nature of the game, the effort $e_{ij}$ for each contestant $i \in N$ and contest $j \in M$ is conditioned on  $\mathbf{a} = (a_i, \mathbf{a}_{-i}) \in A^n$.
Therefore, we focus on finding $\mathbf{a}$ and  $\tilde{e}_{ij}(a_i,\mathbf{a}_{-i})$'s, which are (approximate\footnote{In the absence of closed-form PSNE characterizations, the computation of $\tilde{e}_{ij}(a_i,\mathbf{a}_{-i})$ must be done via numerical methods.}) equilibrium effort of $i\in N$ in $j\in M$ given $\mathbf{a}$. 
Given the dependence of $\tilde{e}_{ij}$'s on $\mathbf{a}$, we simply refer to $\mathbf{a}$ as a self-selection outcome (omitting $\tilde{e}_{ij}$'s), with this abuse of notation adopted henceforth throughout the paper. 
Given a self-selection outcome $\mathbf{a} \in A^n$, the utility of $i \in N$ (Equation \ref{eq:utility}) can hence be redefined as 
\begin{equation}\label{eq:utility-special}
    u_i(a_i, \mathbf{a}_{-i}) = \sum_{j=1}^m a_{ij}\Big(v_{ij} \tilde{p}_{ij}(a_i, \mathbf{a}_{-i}) - c_{ij}(\tilde{e}_{ij}(a_i, \mathbf{a}_{-i}))\Big),
\end{equation}
where $\tilde{p}_{ij}(a_i, \mathbf{a}_{-i}) = \frac{a_{ij} f_{ij}(\tilde{e}_{ij}(a_i, \mathbf{a}_{-i}))}{\sum_{l \in N} a_{lj} f_{lj}(\tilde{e}_{lj}(a_l, \mathbf{a}_{-l}))}$ for any self-selection outcome $\mathbf{a} \in A^n$.

The dissatisfaction for a given contestant $i\in N$ is captured by the standard notion of \emph{regret} \citep{shoham2008multiagent,vlatakis2020no}.
Formally, contestant $i$'s \emph{(ex post) regret} for playing $a_i$ given $\mathbf{a}_{-i}$ is defined as
\begin{equation}\label{eq:rgt-special}
    \text{rgt}_i(a_i,\mathbf{a}_{-i})=\max\{0,\max_{a'_i\in A}u_i(a'_i,\mathbf{a}_{-i}) - u_i(a_i,\mathbf{a}_{-i})\}.
\end{equation}
Intuitively, it is the maximum utility $i$ loses by playing $a_i$.

\begin{definition}\label{def:psne-special}
A self-selection outcome $\mathbf{a}^* = (a_i^*, \mathbf{a}_{-i}^*) \in A^n$ is a pure-strategy Nash equilibrium (PSNE) if and only if for all $i \in N$, for every $a_i \in A$, $u_i(a_i^*, \mathbf{a}_{-i}^*) \geq u_i(a_i, \mathbf{a}_{-i}^*)$.
\end{definition}
When all contestants have zero regret, PSNE is attained. 

\subsection{Platform Objective}\label{subsec:obj}
In the context of crowdsourcing, the contestants' efforts may translate into solutions of different qualities. While the top solution may be announced as the winner, the other submitted solutions may  not be discarded (i.e., different machine learning classifiers can be combined to build a more powerful classifier via some ensemble approaches; see also~\citep{lakhani2013prize}). As a result, the platform seeks many high-quality solutions, which 
can be captured by the overall effort applied to a contest. 
Since contestants are not under contract, the platform 
provides recommendations on which contest(s) they should join based on their demand \citep{difallah2013pick}.
Given a self-selection outcome $\mathbf{a} \in A^n$, the \emph{overall effort} (OE) objective is
\begin{equation}\label{eq:oe}
    OE(\mathbf{a}) = \sum_{i=1}^n \sum_{j=1}^m a_{ij} \tilde{e}_{ij}(a_i, \mathbf{a}_{-i}).
\end{equation}
The platform wants to find
\begin{equation}\label{eq:obj-special}
\begin{aligned}
    \mathbf{a}^{\text{opt}} \in & \argmax_{\mathbf{a} \in A^n} OE(\mathbf{a}) \\
    \text{such that} \quad & u_i(a_i, \mathbf{a}_{-i})\ge 0 \quad \text{for every } i\in N,
\end{aligned}
\end{equation}
where the utility constraints reflect the fact that contestants are rational agents who seek to participate in contests strategically \citep{conlisk1996bounded,jones1999bounded}.

Instead of achieving the highest performance, the platform's objective is to drive a general increase in activity in the field with many high-quality solutions as discussed earlier (e.g., diverse logo and website designs). 
Another example is when the platform wants to maximize the expected learning effort of a certain programming topic or skill made by contestants\footnote{On Topcoder, they form special learning communities around certain groups for the general public \citep{topcoder2018expanding,topcoder2018weeks}.}. We refer readers to~\citep{moldovanu2001optimal}, who also consider the sum of effort as the objective function for their single-contest model, for more examples and motivations.

\section{GRAF: A Greedy Algorithmic Framework for SSTC}\label{sec:graf}

In this section, we present the \textbf{GR}eedy \textbf{A}lgorithmic \textbf{F}ramework (GRAF) for generating self-selection outcomes in SSTC games, followed by establishing its properties in special settings: Zero-regret contestants (Section \ref{subsec:graf-psne}) and platform optimality (Section \ref{subsec:psne-opt}).

Let $\mathcal{BR}_i(\mathbf{a}_{-i}) = \argmax_{a'_i \in A'_i} u_i(a_i+a'_i, \mathbf{a}_{-i})$ be the best-response set of $i$ given the partial self-selection
outcome $\mathbf{a}_{-i} \in A^{n-1}$, where $A'_i=\{a'_i\in\{0,1\}^m\mid\sum_{j\in M:\,a_{ij}=0}a'_{ij}=1\}$ is the single-contest selection set of $i$ concerning contests that $i$ has yet to join.
(We elaborate on how to compute $\tilde{e}_{ij}(a_i,\mathbf{a}_{-i})$ in the supplementary material.)
As shown in Algorithm~\ref{alg:graf}, given an SSTC game where the contestants are ordered based on some score vector $\sigma\in\mathbb{R}^n$, i.e., $\sigma_1\ge\sigma_2\ge\ldots\ge\sigma_n$, GRAF computes the best-response set of $i$ given the selections of contestants up to $i-1$ and commits $i$
to a contest in $i$'s best-response set.
The algorithm terminates when all contestants meet their participation demand $k_i$.

\begin{algorithm}
    \caption{GRAF}
    \label{alg:graf}
    \begin{algorithmic}[1]
        \REQUIRE An SSTC game with $n$ contestants sorted based on a score vector $\sigma\in\mathbb{R}^n$, i.e., $\sigma_1\ge\ldots\ge\sigma_n$
        \ENSURE A self-selection outcome $\mathbf{a}$
        \STATE Let $\mathbf{a} = \mathbf{0}$ \hfill \doc{\# Set $\mathbf{a}$ to be a zero matrix}
        \WHILE{$k_i>0$ for any $i\in N$}
        \FOR{$i = 1, \ldots, n$}
        \IF{$k_i>0$}
        \STATE Let $\mathcal{BR}_i(\mathbf{a}_{-i}) = \argmax_{a'_i \in A'_i} u_i(a_i+a'_i, \mathbf{a}_{-i})$
        \STATE Select $j \in \mathcal{BR}_i(\mathbf{a}_{-i})$, set $a_{ij} = 1$
        \STATE Decrement $k_i$ by 1
        \ENDIF
        \ENDFOR
        \ENDWHILE
    \end{algorithmic}
\end{algorithm}

\subsection{Zero-Regret Contestants}\label{subsec:graf-psne}

Consider games where (1) the contests have a homogeneous type (e.g., either all logo contests, all webpage contests, machine  learning contests) with a similar reward structure, hence each contestant $i\in N$ has the same value $v_i$ for every contest; (2) each contestant selects a single contest, i.e., $k_i=1,\,\forall i\in N$;\footnote{Most contestants are part-time competitors (i.e., the contestants  have a full-time job) where they compete once or twice a month \citep{topcoder2018competing}. Thus, selecting only a single contest within a fixed period (e.g., biweekly) is a feasible assumption.} and (3) both the production and the effort cost of any contestant $i\in N$ in any contest $j\in M$ is linear in effort, i.e., 
$q_{ij}=r_{ij}=c_{ij}=d_{ij}=1$.
In such case, the effort $\tilde{e}_{ij}(a_i,\mathbf{a}_{-i})$ 
can be determined deterministically via the equilibrium characterization of Theorem~\ref{thm:tullock}.
That is---let $C(j) = \{i \in N \mid a_{ij} = 1\}$---$\tilde{e}_{ij}(a_i,\mathbf{a}_{-i})$ for a contestant $i \in C(j)$ is the equilibrium effort $e^*_{ij}$ (via Theorem~\ref{thm:tullock}) restricting to the contestants in $C(j)$. If $i \notin C(j)$, then $e_{ij}^* = 0$.
We refer to this model as simply \emph{special SSTC} when the context is clear.

We begin by adapting Algorithm \ref{alg:graf} to special SSTC games (see Algorithm \ref{alg:psne}) and show that it generates a PSNE (i.e., a no-regret self-selection outcome).
Notably, we order the contestants based on $v_1 \ge v_2 \ge \ldots \ge v_n$ given their homogeneous values $v_i$.
Moreover, since each contestant $i$ selects exactly one contest, which implies $A'_i=A_i=\{a_i \in \{0,1\}^m \mid \sum_{j \in M} a_{ij} = 1\}$, their best-response set can be simplified to $\mathcal{BR}_i(\mathbf{a}_{-i}) = \argmax_{a_i \in A_i} u_i(a_i,\mathbf{a}_{-i})$.

\begin{algorithm}
    \caption{GRAF in Special SSTC}
    \label{alg:psne}
    \begin{algorithmic}[1]
        \REQUIRE A special SSTC game with $n$ contestants sorted based on $\sigma_i=v_i$, i.e., $v_1 \geq \ldots \geq v_n$
        \ENSURE A PSNE self-selection outcome $\mathbf{a}^*$
        \STATE Let $\mathbf{a} = \mathbf{0}$ \hfill \doc{\# Set $\mathbf{a}$ to be a zero matrix}
        \FOR{$i = 1, \ldots, n$}
        \STATE Let $\mathcal{BR}_i(\mathbf{a}_{-i}) = \argmax_{a_i \in A_i} u_i(a_i, \mathbf{a}_{-i})$
        \STATE\label{step:j} Select $j \in \mathcal{BR}_i(\mathbf{a}_{-i})$, set $a_{ij} = 1$
        \ENDFOR
    \end{algorithmic}
\end{algorithm}

\begin{theorem}\label{thm:psne}
For any game of special SSTC, there is a PSNE. 
Moreover, Algorithm \ref{alg:psne} outputs one in polynomial time. 
\end{theorem}

\begin{proof}[Proof Sketch]
The key to showing that Algorithm \ref{alg:psne} returns 
a PSNE is to argue that (1) the partial self-selection outcome $\mathbf{a}_{-i}$ is a PSNE for 
the game played between the contestants up until $i$ 
and (2) the self-selection outcome $(a_i,\mathbf{a}_{-i})$ is still a PSNE 
after setting a best-response of $i$ to be one.
We show this by induction. 
The base step is straightforward, and the inductive step requires 
the following two claims.

\begin{claim}
No contestants in contest $j$ would deviate to a different contest 
after $i$ joining the contest $j$. 
\end{claim}

\begin{claim}
No contestants outside of contest $j$ would deviate 
to contest $j$ after $i$ joining the contest $j$. 
\end{claim}

Along with the above two claims, we have shown 
that Algorithm \ref{alg:psne} finds a PSNE in polynomial time.
\end{proof}

\subsection{Platform Optimality}\label{subsec:psne-opt}

Although GRAF itself does not guarantee optimality in OE, we show that the PSNE self-selection outcome (determined via Algorithm \ref{alg:psne}) is also optimal for a variant of special SSTC.
In particular, in \emph{uniform SSTC} games where the values of contestants are identical, i.e., $v=v_1=\ldots=v_n$, $\mathbf{a}^*=\mathbf{a}^{\text{opt}}$ when the number of contestants is at least twice the number of contests\footnote{In real-world crowdsourcing platforms, the contestant pool is typically much larger than the number of simultaneously active contests \citep{dipalantino2009crowdsourcing,boudreau2011incentives,lakhani2013prize,king2013using}.}.

\begin{theorem}\label{thm:opt}
    In a uniform SSTC game with $n\ge2m$ contestants, every PSNE $\mathbf{a}^*$ achieves optimal OE.
\end{theorem}

\section{LLMScore: An LLM-Based Framework to Augment GRAF}\label{sec:llmscore}

In general, each contestant $i$ may have different values $v_{ij}$ as well as various parameters (i.e., $q_{ij},r_{ij},c_{ij},d_{ij}$) for the contests, rendering the choice of $\sigma$ (which determines the order of contestants to consider in GRAF) nontrivial.
To this end, we propose LLMScore, an LLM-based framework that automates the design of $\sigma$ to augment GRAF.

\paragraph{Overview.}
We start with an overview of our proposed LLMScore framework (illustrated in Figure \ref{fig:llmscore}), which is inspired by the popular EoH method \citep{liu2024evolution} in automated algorithmic design. 
LLMScore maintains a population of $\eta$ heuristic scoring algorithms, denoted as $P=\{h_1,\ldots,h_\eta\}$, where $h_i$ takes an SSTC instance $\mathcal{I}$ as input and outputs a score vector $\sigma\in\mathbb{R}^n$.
It adopts an evolutionary search procedure, iteratively searching for algorithms that yield better trade-offs between OE and regret.
Each algorithm $h_i\in P$ is evaluated on a set of problem instances and assigned a fitness value $g(h_i)$ that indicates its probability of being selected as reference for generating new algorithms.

\begin{figure}[h]
    \centering
    \Description{Schematic of the LLMScore framework: a population of candidate scoring algorithms, each a short program mapping an SSTC instance to a score vector, is evaluated for a fitness that trades off overall effort against regret; the fittest algorithms are fed back to a large language model, which is prompted to generate improved algorithms, repeating in an evolutionary loop.}
    \includegraphics[width=0.4\textwidth]{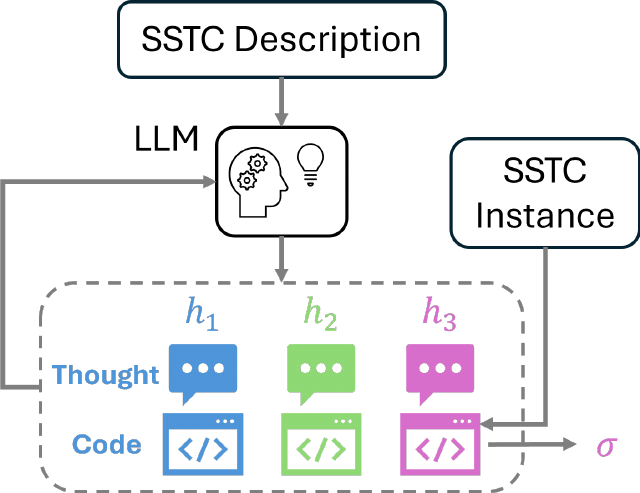}
    \caption{Illustration of the LLMScore framework. (See Figure \ref{fig:prompt-task} for our employed ``SSTC Description'' prompt.)}
    \label{fig:llmscore}
\end{figure}

\paragraph{Algorithm Representation.}
We first describe the components to represent a scoring algorithm $h_i$.
\begin{enumerate}
    \item The algorithm description comprises a few sentences in natural language. It is created by LLMs and encapsulates a high-level thought. An example is provided in Figure \ref{fig:learned-alg1} (second text box).
    \item The code block is an implementation of the algorithm (Figure~\ref{fig:learned-alg1}, bottom). It should follow a pre-defined format (see Listing~\ref{lst:code-template}) so that it can be identified and seamlessly integrated into the evolutionary framework. In the experiments, we choose to implement it as a Python function, though any programming language should work.
    \item Each algorithm is assigned a fitness value to represent its priority in the population, which is used for selection and population management. 
    We elaborate our definition of fitness based on OE and regret shortly.
\end{enumerate}

\paragraph{Evolutionary Search.}
To begin the evolution process, we inform the LLM of the SSTC problem of interest and instruct it to design a scoring algorithm by first presenting the description of the generated algorithm followed by the corresponding code block (see Figure \ref{fig:init_prompt-design}, top).
We repeat $\eta$ times to generate $\eta$ initial algorithms.
From these initial algorithms, offspring algorithms are generated by prompting the LLM with specific instructions, called \emph{prompt strategies}, that either explore or modify existing algorithms.
Exploration strategies focus more on exploring uncharted area in the space of heuristics by conducting crossover-like operators on parent heuristics (Figure \ref{fig:init_prompt-design}, left).
Modification strategies refine a parent heuristic by modifying, adjusting parameters, and removing redundant parts (Figure \ref{fig:init_prompt-design}, right).
For selecting parent algorithm(s) from a population $P$ (consisting of $\eta$ algorithms) when generating an offspring algorithm, we first rank them according to their fitness, then randomly select one or multiple algorithms $h_i$ with probability $\text{Pr}(h_i)\propto1/(\rho_i+\eta)$, where $\rho_i$ is its rank in $P$.
Once $\eta$ offspring algorithms are generated, we select the $\eta$ best algorithms from the current population (comprising parent and offspring algorithms) to form a population for the next iteration.
The evolution process terminates when the evaluation resource (e.g., number of populations) is depleted.

\paragraph{Prompt Refinement.}\label{par:prompt-refine}
Since prompt strategies dictate the diversity in the population and play a vital role in producing high-quality heuristics \citep{liu2024systematic}, thoughtful prompt engineering is necessary, which often requires certain degree of domain knowledge (especially for modification ones).
To alleviate this, we incorporate a prompt refinement module wherein prompt strategies are iteratively refined as well.
In particular, when the fitness values of the top-$\lambda$ algorithms in the population remain unchanged for $t$ consecutive units of evaluation resource, which indicates stagnation during the search, we refine the modification prompt strategies as follows: Given the current set of modification strategies, we ask the LLM to design a new one that is different from them as much as possible (Figure \ref{fig:evol_prompt-design}).
By this means, a new prompt strategy is added after calling this module.
To maintain the size of the set of prompt strategies, during the evolution process, we assign each prompt strategy a score defined as the average fitness value of the top-$\delta$ individuals produced by it, and keep those with highest scores for generating offspring algorithms.
An example of the LLM-generated prompt strategy is given in Figure \ref{fig:learned-alg1} (first text box), which helps produce the algorithm yielding highest OE in one of our runs of LLMScore.

\paragraph{Fitness Evaluation.}\label{par:fitness}

The evaluation of each newly generated algorithm $h$ involves running it on a set of problem instances from which its fitness value is computed.
Since we aim to optimize both OE and regret, the associated fitness function should include both terms.
We achieve this by penalizing the fitness function whenever $h$ yields subpar regret.
Formally, given a heuristic scoring algorithm $h(\mathcal{I})=\sigma$, 
its fitness value $g(h)$ is defined as the median overall effort from the self-selection outcome $\mathbf{a}$ returned by GRAF (Algorithm \ref{alg:graf}): 
\begin{equation}
    g(h)=\text{med}\Big(OE(\mathbf{a})\Big)=\text{med}\Big(OE\big(\mathcal{A}(\mathcal{I};h(\mathcal{I}))\big)\Big),
\end{equation}
where $\mathcal{A}$ denotes Algorithm \ref{alg:graf} parameterized by $\sigma$ and $\text{med}$ is defined over all training instances.
We use the median instead of the mean due to the frequent occurrence of extreme outliers across instances.
Let us define the \emph{max regret} across all contestants as $\max_{i\in N}\text{rgt}_i(a_i,\mathbf{a}_{-i})$.
If the max regret associated with $\mathbf{a}=\mathcal{A}(\mathcal{I};h(\mathcal{I}))$ is higher than the max regret associated with $\mathbf{s}'$ returned by $\mathcal{A}(\mathcal{I};\sigma')$ for some baseline $\sigma'$ in $\tau\times100\%$ or more of the training instances, where $\tau\in[0,1]$, we penalize $h$'s fitness as follows:
\begin{equation}
    \hat{g}(h) = -1/g(h).
\end{equation}
By this means, algorithms yielding sufficiently low max regret are considered more ``fit'' than those otherwise in the population.
If multiple algorithms share the same fitness value, we sort them by median max regret.

\section{Experiments}\label{sec:exp}

\subsection{Experimental Setups}

Throughout the experiments, we chose GPT-4o mini as our pretrained LLM.
Our complete implementation in Python 3.10---including the SSTC instance generator, GRAF, and the LLMScore search---is available at \url{https://github.com/AnonyMouse3005/LLMScore.git}.

\paragraph{Datasets.}\label{par:data}
We generate synthetic SSTC instances where the contestants' values $v_{ij}$ are sampled uniformly from $[5,20]$ and the parameters $q_{ij}$, $r_{ij}$, $c_{ij}$, $d_{ij}$ are sampled from the Gaussian-like distribution $\mathcal{G}=2\times Beta(\alpha=20,\beta=20)\in[0,2]$.
We consider $n\in\{10,15,30,60\}$ and $m\in\{3,5\}$. There are 1,000 instances for each combination of $(m,n)$.
For each instance, we generate $k_i$ with $\text{Pr}(k=1)=0.7$, $\text{Pr}(2)=0.3$ for $m=3$ and $\text{Pr}(k=1)=0.6$, $\text{Pr}(2)=0.3$, $\text{Pr}(3)=0.1$ for $m=5$.

\paragraph{Implementation Details.}
We consider general SSTC games with or without the equilibrium effort characterization (from Theorem \ref{thm:tullock}).
For the former case, we simply set $q_{ij}=r_{ij}=c_{ij}=d_{ij}=1$.
In either case, contestants may select exactly one contest ($k_i=1,\,\forall i\in N$) or multiple contests ($k_i\ge1$).
For clarity, we label each problem setting (ordered by increasing complexity) as follows. 
\begin{itemize}
    \item \textbf{SSTC 1}: Single-selection SSTC with characterization.
    \item \textbf{SSTC 2}: Multi-selection SSTC with characterization.
    \item \textbf{SSTC 3}: Single-selection SSTC without characterization.
    \item \textbf{SSTC 4}: Multi-selection SSTC without characterization.
\end{itemize}

Notably, we train LLMScore only for SSTC 3 using a sizable training set of 50 small instances (with $n=10$ contestants and $m=3$ contests), then use the learned algorithms for $\sigma$ to test on instances from all other settings.
We choose SSTC~3 for its generality (with respect to SSTC~1 and SSTC 2) without being too complex (as in SSTC 4) for standard LLMs.
Once training is done, we compute the median OE and the median max regret (across 1,000 aforementioned instances) from (i) the algorithm with highest median OE and (ii) the one with lowest median max regret across the training instances.
We then repeat for three runs of LLMScore and report the average performance.
(Further implementation details in the supplementary material.)

\paragraph{Baselines.}\label{par:baselines}

Naturally, we consider GRAF with multiple handcrafted $\sigma$ as baselines:
\begin{itemize}
    \item $\sigma_i^{\text{ID}}=1/i$, i.e., by contestants' IDs, which is also employed as $\sigma'$ during \nameref{par:fitness}.
    \item $\sigma_i^{\texttt{v\_max}}=\max_{j\in M}v_{ij}$.
    \item $\sigma_i^{\texttt{v\_sum}}=\sum_{j\in M}v_{ij}$.
\end{itemize}
For SSTC 3 and SSTC 4, which feature heterogeneous values for $q_{ij},r_{ij},c_{ij},d_{ij}$, we consider more choices for $\sigma$:
\begin{itemize}
    \item $\sigma_i^{\texttt{e\_max}}=\max_{j\in M}e^{\max}_{ij}$, where $e^{\max}_{ij}=(\frac{v_{ij}}{c_{ij}})^{1/d_{ij}}$.
    \item $\sigma_i^{\texttt{e\_sum}}=\sum_{j\in M}e^{\max}_{ij}$.
    \item $\sigma_i^{\texttt{f\_max}}=\max_{j\in M}f_{ij}(e^{\max}_{ij})$.
    \item $\sigma_i^{\texttt{f\_sum}}=\sum_{j\in M}f_{ij}(e^{\max}_{ij})$.
\end{itemize}

In addition to GRAF, we also include \emph{iterated best response} (IBR) \citep{fudenberg1998theory,ho1998iterated}, detailed in the supplementary material.
As shown shortly in our experimental results, IBR is a surprisingly strong baseline that yields high OE while maintaining zero max regret in all instances of SSTC 1 and nearly all instances of SSTC 2. 

\subsection{Results}

\begin{table*}[h]
    \centering
    \small
    \caption{\textbf{SSTC with characterization.} The values in each cell respectively record median OE and median max regret. The best and second-best results are bolded and underlined, respectively.}
    \begin{subtable}[t]{0.97\textwidth}
        \centering
        \caption{SSTC 1. The median OE of the optimal outcomes in n10m3 (via brute-forcing the space of $\mathbf{a}$) is 32.934.}
        \begin{tabular}{l|cc|cc|cc|cc|cc|cc}
            \thickhline
            \multicolumn{1}{c|}{Method} & \multicolumn{2}{c|}{n10m3} & \multicolumn{2}{c|}{n15m3} & \multicolumn{2}{c|}{n30m3} & \multicolumn{2}{c|}{n60m3} & \multicolumn{2}{c|}{n30m5} & \multicolumn{2}{c}{n60m5} \\\hline
            IBR & \textbf{32.636} & \textbf{0.000} & \textbf{38.377} & \textbf{0.000} & \textbf{45.254} & \textbf{0.000} & \textbf{49.893} & \textbf{0.000} & \textbf{72.156} & \textbf{0.000} & \textbf{81.741} & \textbf{0.000} \\\hline
            GRAF[$\sigma^{\text{ID}}$] & 30.943 & 0.997 & 37.328 & 0.723 & 44.674 & 0.374 & 49.672 & 0.178 & 70.181 & 0.893 & 80.826 & 0.410 \\
            GRAF[$\sigma^{\texttt{v\_max}}$] & 30.943 & 0.997 & 37.328 & 0.723 & 44.674 & 0.374 & 49.672 & 0.178 & 70.181 & 0.893 & 80.826 & 0.410 \\
            GRAF[$\sigma^{\texttt{v\_sum}}$] & 30.943 & 0.997 & 37.328 & 0.723 & 44.674 & 0.374 & 49.672 & 0.178 & 70.181 & 0.893 & 80.826 & 0.410 \\\hline
            GRAF[$\sigma^{\text{LLM}}$] (best OE) & \underline{31.113} & 0.889 & \underline{37.528} & \underline{0.399} & \underline{45.095} & \underline{0.071} & \underline{49.874} & 0.009 & \underline{70.339} & 0.695 & \underline{81.268} & \underline{0.196} \\
            GRAF[$\sigma^{\text{LLM}}$] (best regret) & 31.084 & \underline{0.858} & \underline{37.528} & \underline{0.399} & \underline{45.095} & \underline{0.071} & 49.872 & \underline{0.009} & 70.307 & \underline{0.694} & \underline{81.268} & \underline{0.196} \\
            \thickhline
        \end{tabular}
        \label{subtab:special-k1}
    \end{subtable}\\
    \medskip
    \begin{subtable}[t]{0.97\textwidth}
        \centering
        \caption{SSTC 2. The median OE of the optimal outcomes in n10m3 (via brute-forcing the space of $\mathbf{a}$) is 34.689.}
        \begin{tabular}{l|cc|cc|cc|cc|cc|cc}
            \thickhline
            \multicolumn{1}{c|}{Method} & \multicolumn{2}{c|}{n10m3} & \multicolumn{2}{c|}{n15m3} & \multicolumn{2}{c|}{n30m3} & \multicolumn{2}{c|}{n60m3} & \multicolumn{2}{c|}{n30m5} & \multicolumn{2}{c}{n60m5} \\\hline
            IBR & \textbf{34.638} & \textbf{0.017} & \textbf{39.345} & \textbf{0.010} & \underline{45.709} & \textbf{0.004} & \underline{50.042} & \textbf{0.000} & \textbf{74.522} & \textbf{0.003} & \textbf{82.767} & \textbf{0.000} \\\hline
            GRAF[$\sigma^{\text{ID}}$] & 33.758 & 0.549 & 38.954 & 0.443 & 45.304 & 0.230 & 49.841 & 0.114 & 73.512 & 0.626 & 82.331 & 0.285 \\
            GRAF[$\sigma^{\texttt{v\_max}}$] & 33.528 & 0.648 & 38.855 & 0.474 & 45.403 & 0.251 & 49.975 & 0.114 & \underline{73.954} & 0.577 & 82.458 & 0.292 \\
            GRAF[$\sigma^{\texttt{v\_sum}}$] & 33.528 & 0.648 & 38.855 & 0.474 & 45.403 & 0.251 & 49.975 & 0.114 & \underline{73.954} & 0.577 & 82.458 & 0.292 \\\hline
            GRAF[$\sigma^{\text{LLM}}$] (best OE) & \underline{34.095} & 0.549 & \underline{39.124} & \underline{0.182} & \textbf{45.713} & \underline{0.006} & \textbf{50.071} & 0.063 & 73.818 & \underline{0.414} & \underline{82.620} & \underline{0.122} \\
            GRAF[$\sigma^{\text{LLM}}$] (best regret) & 34.000 & \underline{0.409} & \underline{39.124} & \underline{0.182} & \textbf{45.713} & \underline{0.006} & \textbf{50.027} & \underline{0.004} & 73.818 & \underline{0.414} & \underline{82.620} & \underline{0.122} \\
            \thickhline
        \end{tabular}
        \label{subtab:special-k}
    \end{subtable}
    \label{tab:sstc-special}
\end{table*}

\paragraph{SSTC With Characterization.}
Table \ref{tab:sstc-special} shows results for SSTC 1 (top) and SSTC 2 (bottom).
Across nearly all settings, $\text{GRAF}[\sigma^{\text{LLM}}]$ outperforms all GRAF variants with handcrafted $\sigma$ in both OE and max regret, and
even surpasses IBR in OE in several larger settings (e.g., n30m3 and n60m3 in Table~\ref{subtab:special-k}).
While IBR undoubtedly yields the best results in SSTC 1, we notice the gaps in its performance with respect to GRAF-based methods are substantially reduced in SSTC 2, which could be attributed to the more complex game structures in the presence of multiple-contest selections.
For GRAF, interestingly, the more intuitive scores $\sigma^{\texttt{v\_max}},\sigma^{\texttt{v\_sum}}$ and the simplistic score $\sigma^{\text{ID}}$ yield almost identical results, underscoring the nontriviality of designing $\sigma$ in general SSTC games.
Incidentally, while the LLM-generated scoring algorithms were trained exclusively on small instances of SSTC~3 (n10m3), they transfer effectively to instances of SSTC~1 and SSTC~2 that differ in both problem structure and scale, which highlights the generalizability of LLMScore.

\paragraph{SSTC Without Characterization.}

Table \ref{tab:sstc-general} shows results for SSTC 3 (top) and SSTC 4 (bottom).
In these more challenging settings,
$\text{GRAF}[\sigma^{\text{LLM}}]$ outperforms \emph{all} baselines, including IBR, in both OE and max regret across all settings and instance sizes.
Given the heterogeneity of various contestant-specific parameters, all handcrafted scores beside $\sigma^{\texttt{f\_max}}$ and $\sigma^{\texttt{f\_sum}}$ yield identical results, and the improvement between $\sigma^{\texttt{f\_max}}$, $\sigma^{\texttt{f\_sum}}$ over $\sigma^{\text{ID}}$ in either OE or max regret is only observed in small instances (n10m3 and n15m3).
This demonstrates that intuitive designs of $\sigma$ are insufficient in the most general settings of SSTC.

\begin{table*}[h]
    \centering
    \small
    \caption{\textbf{SSTC without characterization.} The values in each cell respectively record median OE and median max regret. The best and second-best results are bolded and underlined, respectively. Methods with identical results are condensed.}
    \begin{subtable}[t]{0.97\textwidth}
        \centering
        \caption{SSTC 3.}
        \begin{tabular}{l|cc|cc|cc|cc|cc|cc}
            \thickhline
            \multicolumn{1}{c|}{Method} & \multicolumn{2}{c|}{\cellcolor{gray!35}n10m3} & \multicolumn{2}{c|}{n15m3} & \multicolumn{2}{c|}{n30m3} & \multicolumn{2}{c|}{n60m3} & \multicolumn{2}{c|}{n30m5} & \multicolumn{2}{c}{n60m5} \\\hline
            IBR & 28.533 & \underline{1.652} & 36.859 & 2.138 & 51.102 & 1.721 & 67.417 & 1.400 & \underline{83.857} & \underline{2.358} & 108.289 & \underline{1.947} \\\hline
            GRAF[$\sigma^{\text{ID}}$/$\sigma^{\texttt{v\_max}}$/$\sigma^{\texttt{v\_sum}}$/$\sigma^{\texttt{e\_max}}$/$\sigma^{\texttt{e\_sum}}$] & 24.966 & 1.876 & 34.470 & 1.960 & \underline{54.522} & 1.770 & \underline{72.497} & 1.783 & 78.442 & 2.932 & \underline{111.633} & 2.326 \\
            GRAF[$\sigma^{\texttt{f\_max}}$] & 26.881 & 2.792 & 35.436 & 2.731 & 46.569 & 2.734 & 60.500 & 2.273 & 73.821 & 4.012 & 93.985 & 3.707 \\
            GRAF[$\sigma^{\texttt{f\_sum}}$] & 26.871 & 2.575 & 37.457 & 2.984 & 48.481 & 2.643 & 63.211 & 2.016 & 73.985 & 4.392 & 97.885 & 3.648 \\\hline
            GRAF[$\sigma^{\text{LLM}}$] (best OE) & \textbf{33.708} & 2.226 & \textbf{42.064} & \underline{1.584} & \textbf{55.967} & \underline{0.642} & \textbf{73.445} & \underline{1.011} & \textbf{84.813} & 2.494 & \textbf{113.877} & \textbf{1.205} \\
            GRAF[$\sigma^{\text{LLM}}$] (best regret) & \underline{32.929} & \textbf{1.576} & \underline{41.167} & \textbf{1.097} & 51.970 & \textbf{0.484} & 72.324 & \textbf{0.716} & 82.511 & \textbf{1.565} & \textbf{113.877} & \textbf{1.205} \\
            \thickhline
        \end{tabular}
        \label{subtab:general-k1}
    \end{subtable}\\
    \medskip
    \begin{subtable}[t]{0.97\textwidth}
        \centering
        \caption{SSTC 4. n60m5 was omitted after IBR failed to complete all test instances within a week.}
        \begin{tabular}{l|cc|cc|cc|cc|cc}
            \thickhline
            \multicolumn{1}{c|}{Method} & \multicolumn{2}{c|}{n10m3} & \multicolumn{2}{c|}{n15m3} & \multicolumn{2}{c|}{n30m3} & \multicolumn{2}{c|}{n60m3} & \multicolumn{2}{c}{n30m5} \\\hline
            IBR & 31.265 & 1.404 & 42.193 & 1.706 & 53.652 & 0.769 & 64.925 & \underline{0.982} & 87.592 & 1.594 \\\hline
            GRAF[$\sigma^{\text{ID}}$/$\sigma^{\texttt{v\_max}}$/$\sigma^{\texttt{v\_sum}}$/$\sigma^{\texttt{e\_max}}$/$\sigma^{\texttt{e\_sum}}$] & 27.560 & 1.450 & 41.152 & 1.408 & 57.699 & 1.646 & \underline{73.916} & 1.651 & 86.627 & 2.054 \\
            GRAF[$\sigma^{\texttt{f\_max}}$] & 30.491 & 1.643 & 41.213 & 1.455 & 50.529 & 1.711 & 62.872 & 1.812 & 83.083 & 2.741 \\
            GRAF[$\sigma^{\texttt{f\_sum}}$] & 29.125 & \underline{0.887} & 42.022 & 1.834 & 50.529 & 1.373 & 63.264 & 1.539 & 82.404 & 3.172 \\\hline
            GRAF[$\sigma^{\text{LLM}}$] (best OE) & \textbf{35.749} & 1.166 & \textbf{48.622} & \underline{0.813} & \textbf{59.506} & \underline{0.624} & \textbf{73.975} & 1.085 & \textbf{90.387} & \underline{1.339} \\
            GRAF[$\sigma^{\text{LLM}}$] (best regret) & \underline{35.517} & \textbf{0.816} & \underline{45.128} & \textbf{0.762} & \underline{57.713} & \textbf{0.514} & 71.150 & \textbf{0.771} & \underline{89.700} & \textbf{1.249} \\
            \thickhline
        \end{tabular}
        \label{subtab:general-k}
    \end{subtable}
    \label{tab:sstc-general}
\end{table*}

Figures \ref{fig:box-ssstc_general}--\ref{fig:tradeoff-ssstc_general} further confirm our findings: The LLM-augmented algorithms dominate both IBR and $\text{GRAF}$ with handcrafted $\sigma$
across all settings of SSTC 3, with the top 10 (by fitness at the end of LLMScore training) consistently forming the approximate Pareto front in the OE-regret trade-off space.
(Analogous results for SSTC 4 in Figures \ref{fig:box-sstc_general} and \ref{fig:tradeoff-sstc_general}.)
Crucially, these impressive performances are achieved through scoring algorithms trained only on small instances of SSTC 3 (n10m3) and applied directly---without any retraining---to larger instances and structurally more complex settings (SSTC 4), once again underscoring the generalizability of LLMScore.

\paragraph{Analysis of LLM-Generated Scores.}
Given the interpretability of LLMScore via code, we analyze the scores $\sigma^{\text{LLM}}$ behind its good results.
Figure \ref{fig:learned-alg1} illustrates a learned scoring algorithm from ``$\text{GRAF}[\sigma^{\text{LLM}}]$ (best OE)'' in Tables \ref{tab:sstc-special} and \ref{tab:sstc-general}.
From the code block, the score for contestant $i\in N$ can be formally expressed as $\sigma_i=\phi_i\sum^m_{j=1}[\hat{v}_{ij}\ln(p_{ij}(e_{ij}))-\ln(c_{ij}(e_{ij}))]$, where ``adjusted valuation'' $\hat{v}_{ij}=v_{ij}\cdot\frac{n-i+1}{n}$ and ``rank bonus'' $\phi_i=\frac{n-i}{n}$. The effort $e_{ij}$ for defining $p_{ij}(e_{ij})=\frac{f_{ij}(e_{ij})}{\sum_{l\in N}f_{lj}(e_{lj})}$ and $c_{ij}(e_{ij})$ is $e_{ij}=\frac{n-i}{(n+1)^2}$ for any $j\in M$.
(Please refer to the caption of Figure \ref{fig:learned-alg3} for ``$\text{GRAF}[\sigma^{\text{LLM}}]$ (best regret)''.)
We see that the scores in both examples are substantially more sophisticated than the handcrafted ones from \nameref{par:baselines}.

\paragraph{Runtime Analysis.}

Figure \ref{fig:runtime} compares the median runtime of IBR against GRAF (with $\sigma^{\text{ID}}$ and $\sigma^{\text{LLM}}$) across instances of SSTC 4 (the most complex setting among the four) on a logarithmic scale.
We see that GRAF with either score is noticeably faster than IBR: A single run of IBR (i.e., one randomly-generated initialization of $\mathbf{a}$) takes at least $10\times$ the time of $\text{GRAF}[\sigma^{\text{ID}}]$ or $\text{GRAF}[\sigma^{\text{LLM}}]$ across all test instance sizes, and the gap widens with $n$.
Note that since the training of LLMScore is conducted offline and amortized over all future test instances, the per-instance cost of $\text{GRAF}[\sigma^{\text{LLM}}]$ at test time is almost identical to that of $\text{GRAF}[\sigma^{\text{ID}}]$.
Combined with its generalizability, LLMScore offers a scalable method in SSTC.

\begin{figure}[h]
    \centering
    \Description{Line plot with a logarithmic vertical axis comparing the median per-instance runtime in seconds of IBR, GRAF with the identity score, and GRAF with the LLMScore-designed score as the SSTC 4 instance size grows; both GRAF variants run at least ten times faster than IBR, and the gap widens as the number of contestants increases.}
    \includegraphics[width=0.4\textwidth]{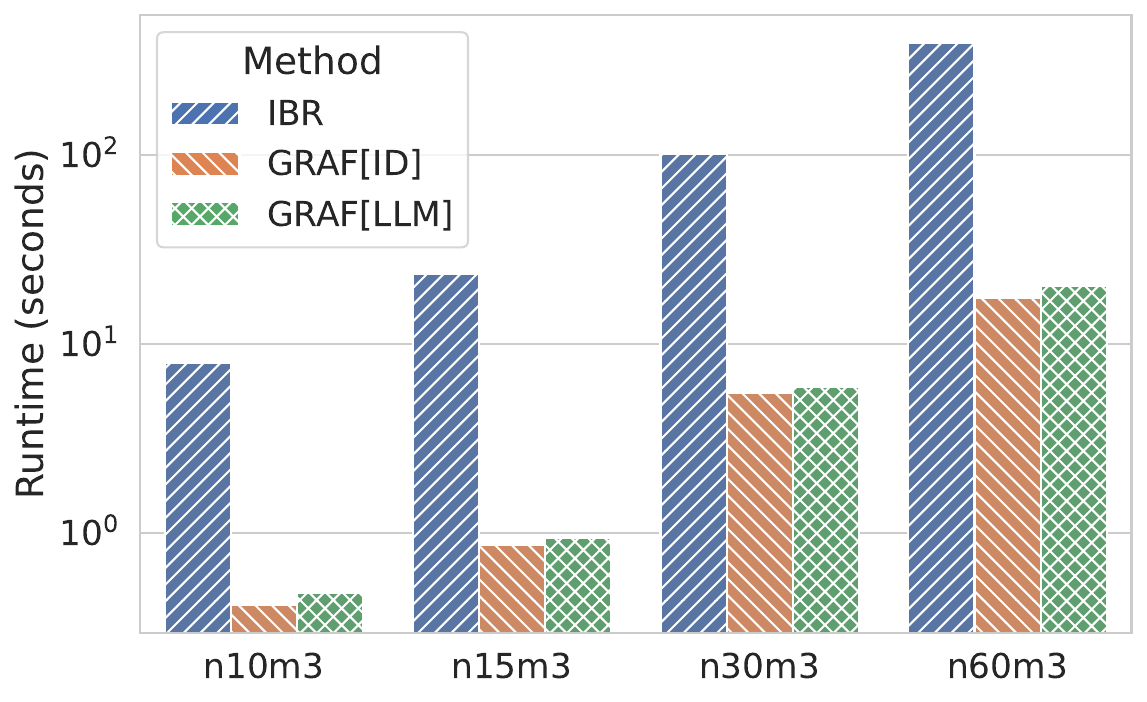}
    \caption{Median runtime (seconds in logarithmic scale) across instances of SSTC 4.
    }
    \label{fig:runtime}
\end{figure}

\section{Related Work}\label{sec:related}
\paragraph{Crowdsourcing and Worker Self-Selection.}
Crowdsourcing contests are a well-studied mechanism for eliciting effort from online workers~\citep{dipalantino2009crowdsourcing,lakhani2013prize,king2013using,cricelli2022crowdsourcing}, and a recurring theme is that workers \emph{self-select} which tasks to undertake rather than being assigned, which makes coordinating the crowd a central challenge. A large body of human-computation work therefore studies how to guide that self-selection through task recommendation and assignment~\citep{difallah2013pick} and incentive design~\citep{boudreau2011incentives}. LLMs, too, are increasingly used in crowdsourcing for laborious tasks such as data annotation~\citep{he2024annollm}, misinformation assessment~\citep{xu2024role}, and content creation~\citep{ma2025large}. Our contribution is complementary: Rather than designing the contests or rewards, or performing the tasks, we take the contests as given and study how a platform should \emph{recommend} where heterogeneous workers compete, optimizing platform value and worker satisfaction jointly.

\paragraph{Self-Selection in Tullock Contests.}
A Tullock contest models a rent-seeking situation in which contestants exert effort to win an indivisible reward, with the winner determined probabilistically~\citep{tullock1980efficient,anne1974political}. The PSNE characterization of a single contest (Theorem~\ref{thm:tullock}) is due to \citet{hillman1989politically,stein2002asymmetric}. \citet{he2024complexity} show that the effort-elasticity parameter $r_i$ governs hardness: With more than $O(\log n)$ contestants and $r_i\in(1,2]$, deciding PSNE existence is NP-complete, so no characterization exists---a regime our SSTC~3 and SSTC~4 instances ($r_{ij}\in[0,2]$) deliberately capture. Table~\ref{tab:models} compares our setting with the two existing SSTC works~\citep{bernergard2017self,cohen2026algorithmic}, both of which assume like-minded contestants ($v_j=v_{1j}=\cdots=v_{nj}$) and give partial PSNE characterizations only under restrictive assumptions (linear or identity production and cost functions, single-contest selection). To our knowledge, general SSTC with multiple-contest selection by heterogeneous contestants with arbitrary nonlinear production and costs remains uncharted.

\begin{table}[h]
    \centering
    \small
    \caption{SSTC model comparison. Parameters: $c,d$---effort cost parameters, $q$---production efficiency, $r$---effort elasticity, $v$---personal value, $k$---participation demand. (\textdagger)---Weak convexity assumption for $c_{ij}(e)$, hence $d_{ij}\ge1$; (*)---Weak concavity assumption for $f_{ij}(e)=q_{ij}\cdot e^{r_{ij}}$, hence $r_{ij}\le1$. \citet{bernergard2017self,cohen2026algorithmic} do not conduct empirical studies. In our experiments, $\mathcal{G}=2\times Beta(20,20)\in[0,2]$.}
    \begin{tabular}{l|c|c|c|c}
        & \citep{bernergard2017self} & \citep{cohen2026algorithmic} & Ours (special) & Ours (general) \\\thickhline
        $c,d$
        & $c_{ij}=d_{ij}=1$
        & $c_{ij};\;d_{ij}\ge1$ (\textdagger)
        & $c_{ij}=d_{ij}=1$
        & $c_{ij}\sim \mathcal{G};\;d_{ij}\sim \mathcal{G}$ \\\hline
        $q$
        & $q_{ij}=1$
        & $q_{ij}$
        & $q_{ij}=1$
        & $q_{ij}\sim \mathcal{G}$ \\
        $r$
        & $r_{ij}=1$
        & $r_{ij}\in(0,1]$ (*)
        & $r_{ij}=1$
        & $r_{ij}\sim \mathcal{G}$ \\\hline
        $v$
        & $v_j$
        & $v_j$
        & $v_i$
        & $v_{ij}\sim \mathcal{U}(5,20)$ \\\hline
        $k$
        & $k_i=1$
        & $k_i=1$
        & $k_i=1$
        & $1\le k_i<m$ \\
    \end{tabular}
    \label{tab:models}
\end{table}

\paragraph{Learning Nash Equilibria.}
A growing body of work computes Nash equilibria via reinforcement learning, e.g., Nash-DQN~\citep{casgrain2022nashdqn}, CCGnet~\citep{wu2023ccgnet}, and RENES~\citep{wang2024renes}. These target mixed-strategy equilibria in small games (at most five players) and, like most deep methods, lack interpretability.

\paragraph{Game Theory and LLMs.}
Work bridging game theory and LLMs largely applies game-theoretic concepts to evaluate LLMs' strategic behavior~\citep{akata2025playing}, interpret them~\citep{horovicz2024tokenshap}, align them with human preferences~\citep{conitzer2024position,swamy2024minimaximalist}, and characterize their societal impact~\citep{fish2023generative,yao2024human} (see~\citep{sun2025survey} for a survey). Others use LLMs to simulate human agents in behavioral economics~\citep{horton2023large,karten2025llm}. Applying LLMs to solve challenging games algorithmically, however, remains largely unexplored.

\paragraph{LLMs for Algorithmic Design.}
LLMs have been applied to mathematical reasoning~\citep{ahn2024large}, code generation~\citep{jiang2024survey}, scientific discovery~\citep{wang2023scientific}, and algorithm design~\citep{liu2024systematic}: FunSearch~\citep{romera2024mathematical} evolves mathematical functions, EoH~\citep{liu2024evolution} evolves heuristics, and ADAS~\citep{hu2024automated} designs agentic systems, with iterative LLM-in-the-loop search shown valuable for evaluation-heavy design tasks~\citep{zhang2024understanding}. Applying such methods to SSTC requires resolving the two challenges from our introduction: We address Challenge~I with a fitness-penalizing scheme combining OE and regret into one objective, and Challenge~II by training only on small instances of a single SSTC setting and reusing the learned algorithms elsewhere---whose demonstrated generalizability validates our approach.

\section{Conclusion}\label{sec:conclusion}
We study how a crowdsourcing platform can guide worker self-selection through recommendations. We introduce GRAF, a greedy framework that, for special cases of SSTC, provably yields outcomes with zero worker regret and platform optimality, and---when augmented by our LLMScore framework---generates high-quality or even near-optimal self-selection outcomes with low worker dissatisfaction in general settings. Because LLMScore outputs human-readable code, the recommendation logic it learns is transparent enough for a platform operator to inspect, audit, and adjust rather than treat as a black box. Together, these let a platform offer advice that benefits the platform and its workers alike.

\paragraph{Limitations and Future Work.}\label{par:limitations}
Our evaluation is synthetic: Workers are modeled as utility-maximizing agents with parameters drawn from chosen distributions rather than estimated from a live platform, so the modeled regret and overall effort are proxies for---not measurements of---real worker dissatisfaction and solution quality. We nonetheless regard regret as behaviorally meaningful: A worker who repeatedly enters contests they would not have chosen with better advice is at risk of disengaging, so driving regret toward zero aligns with worker retention. We thus position this work as the algorithmic foundation on which human-grounded studies can build---e.g., calibrating instances against real platform data (such as Topcoder participation, skill, and prize distributions~\citep{lakhani2013prize,boudreau2011incentives}) and validating recommendations with crowdworkers in the loop.
On the technical side, our effort solver (Brent's algorithm~\citep{brent2013algorithms}) may return local optima, and PSNE existence in fully general SSTC remains open.
Finally, the LLM-generated scoring algorithms may process the contestant index $i$ such that a relabeling of an otherwise identical instance can yield a different greedy order and hence different recommendations.
Future work could consider canonicalizing the raw index \citep{lynch1999canonicalization} to ensure permutation equivariance.

\begin{acks}
Hau Chan is supported by the National Science Foundation under grants IIS:RI \#2302999 and IIS:RI \#2414554. Nguyen Thach is supported by the National Science Foundation under grant IIS:RI \#2414554. Part of this work was conducted while Hau Chan was a Postdoctoral Fellow at the Laboratory for Innovation Science at Harvard. The content is solely the responsibility of the authors and does not necessarily represent the official views of the funding agencies.
\end{acks}

\section*{Disclosure of LLM Use}
The LLMScore framework (Section~\ref{sec:llmscore}) queries LLMs to generate and refine GRAF's scoring functions, a central contribution of this work.
Beyond that, we use LLMs to assist with copyediting and polishing the writing of this manuscript.
All technical content, results, and analyses are the authors' own creation; LLMs were not used elsewhere in the research.

\section*{Ethical Considerations}
This study is entirely computational: All experiments are run on synthetic SSTC instances, and no human subjects, personal data, or live platform interactions are involved. The work is motivated by improving outcomes for both crowdsourcing platforms and the workers on them, and we view low-regret recommendations as a way to reduce worker dissatisfaction rather than to extract more effort against workers' interests. We note that any deployment should respect worker autonomy---recommendations are advisory and constrained by each worker's own demand $k_i$---and should be validated against real worker behavior before being relied upon, a point we return to in our \nameref{par:limitations}.

\bibliographystyle{ACM-Reference-Format}
\bibliography{main}

\renewcommand{\thefigure}{S\arabic{figure}}
\renewcommand{\thetable}{S\arabic{table}}
\renewcommand{\theequation}{S\arabic{equation}}
\renewcommand{\thelisting}{S\arabic{listing}}

\appendix




\begin{figure*}[h]
    \centering
    \Description{Text box containing the natural-language task-description prompt given to the LLM, which explains the SSTC setting (contestants, contests, effort, and winning probability) and instructs the model to produce a scoring algorithm that balances overall effort against regret.}
    \begin{subfigure}[h]{0.95\textwidth}
        \centering
        \footnotesize
        \begin{promptbox}{\textwidth}{Task Description}
        Given N players and M contests (N > M), the task is to find an ex ante priority score for each player in order to greedily add them to one of the contests.

        \hphantom\\

        **Problem Context:**

        - Each player has a valuation for winning a specific contest and a set of parameters affecting winning probability and cost

        - Each player seeks to participate in one contest

        - All contests are empty initially, and players iteratively enter one of them based on their priority scores

        - Once entered, players in a contest exert costly effort depending on existing players

        - The utility of player i participating in contest j is (valuation for j × winning probability) - cost

        - Both winning probability and cost are proportional to the effort put in (see below)

        - Goal: Maximize social welfare, defined as the sum of efforts from all players
        \end{promptbox}
    \end{subfigure}
    \caption{Task description used in LLM prompts. ``Players'' here refer to contestants.}
    \label{fig:prompt-task}
    \vspace{30pt}
\end{figure*}


\begin{figure*}[h]
    \centering
    \Description{Text boxes showing the three prompt templates---initialization, exploration, and modification---used to make the LLM generate, diversify, and refine candidate scoring algorithms; the interchangeable prompt-strategy portions are highlighted in purple.}
    \begin{subfigure}[h]{0.8\textwidth}
        \centering
        \footnotesize
        \begin{promptbox}{\textwidth}{Prompt for Initialization}
        \textcolor{cyan}{[Task Description]}

        \hphantom\\

        First, describe your new algorithm and main steps in one sentence. The description must be inside a brace. Next, implement it in Python using the following template:

        \hphantom\\

        \textcolor{dkgreen}{[Code Template]}

        \hphantom\\
        
        Do not give additional explanations and do not use additional other packages.
        \end{promptbox}
    \end{subfigure}\\
    \begin{subfigure}[h]{0.45\textwidth}
        \centering
        \footnotesize
        \begin{promptbox}{\textwidth}{Prompt for Exploration (E)}
        \textcolor{cyan}{[Task Description]}

        \hphantom\\

        I have 2 existing algorithms with their codes as follows:\\
        No. 1 algorithm and the corresponding code are:\\
        \textcolor{blue}{[Algorithm 1 Description]}\\
        \textcolor{dkgreen}{[Code 1]}

        \hphantom\\

        No. 2 algorithm and the corresponding code are:\\
        \textcolor{blue}{[Algorithm 2 Description]}\\
        \textcolor{dkgreen}{[Code 2]}

        \hphantom\\

        \textcolor{mauve}{Please help me create a new algorithm that has a totally different form from the given ones.}\\
        First, describe your new algorithm and main steps in one sentence. The description must be inside a brace. Next, implement it in Python using the following template:

        \hphantom\\

        \textcolor{dkgreen}{[Code Template]}

        \hphantom\\
        
        Do not give additional explanations and do not use additional other packages.
        \end{promptbox}
    \end{subfigure}%
    ~
    \begin{subfigure}[h]{0.5\textwidth}
        \centering
        \footnotesize
        \begin{promptbox}{\textwidth}{Initial Prompt for Modification (M1)}
        \textcolor{cyan}{[Task Description]}

        \hphantom\\

        I have one algorithm with its code as follows:\\
        Algorithm description:  \textcolor{blue}{[Algorithm Description]}\\
        Code: \textcolor{dkgreen}{[Code]}\\
        Welfare: \textcolor{brown}{[Fitness Value]}

        \hphantom\\

        \textcolor{mauve}{If welfare is positive, please identify the main algorithm parameters and assist me in creating a new algorithm that has a different parameter settings of the score function provided. Otherwise, if welfare is negative, please help me revise the algorithm to improve its efficiency, which means that anytime a player i is added to a contest, no other players i-1 would gain in utility by deviating to another contest.}\\
        First, describe your new algorithm and main steps in one sentence. The description must be inside a brace. Next, implement it in Python using the following template:

        \hphantom\\

        \textcolor{dkgreen}{[Code Template]}

        \hphantom\\
        
        Do not give additional explanations and do not use additional other packages.
        \end{promptbox}
    \end{subfigure}
    \caption{Prompts used for initialization, exploration, and modification. The prompt strategies are marked in purple. See Figure \ref{fig:prompt-task} for ``Task Description'' and Listing \ref{lst:code-template} for the designed ``Code Template''. Note that by ``improving efficiency'', we mean reducing regret.}
    \label{fig:init_prompt-design}
\end{figure*}


\begin{figure*}[h]
    \centering
    \begin{subfigure}[h]{0.9\textwidth}
        \centering
        \footnotesize
        \begin{promptbox}{\textwidth}{Prompt for Evolving Modification Prompt Strategies}
        \textcolor{cyan}{[Task Description]}

        \hphantom\\

        I want to leverage the capabilities of LLMs to generate heuristic algorithms that can efficiently tackle this problem. I have already developed a set of initial prompts and observed the corresponding outputs. However, to improve the effectiveness of these algorithms, we need your assistance in carefully analyzing the existing prompts and their results. Based on this analysis, we ask you to generate new prompts that will help us achieve better results in solving the problem.

        \hphantom\\

        I have X existing prompts with average score (the higher the better) as follows:\\
        No. 1 prompt:\\
        Content: \textcolor{mauve}{[Prompt Strategy]}\\
        Score: \textcolor{orange}{[Prompt Fitness Value]}\\
        ...\\
        No. X prompt:\\
        Content: \textcolor{mauve}{[Prompt Strategy]}\\
        Score: \textcolor{orange}{[Prompt Fitness Value]}

        \hphantom\\

        Note that these prompts ask LLMs to refine an input strategy by modifying, adjusting parameters, and removing redundant parts.

        \hphantom\\

        Please help me create a new prompt that has a totally different form from the given ones but can be motivated from them. Describe your new prompt and main steps in one sentence. The description must be inside a brace. Do not give additional explanations.
        \end{promptbox}
    \end{subfigure}
    \Description{Text box showing the meta-prompt used to evolve the prompt strategies themselves: it takes the current set of prompt strategies as input and asks the LLM to propose new ones.}
    \caption{Prompt used for evolving prompt strategies, where X is the size of the current set of prompt strategies.}
    \label{fig:evol_prompt-design}
    \vspace{50pt}
\end{figure*}



\begin{listing*}[b]%
\caption{Code template}%
\label{lst:code-template}%
\begin{lstlisting}[language=Python]
import numpy as np

def priority(valuations, win_prob_params, cost_params):
    '''
    Args:
        valuations (np.ndarray): N-by-M matrix where entry at (i, j) denotes valuation of player i for winning contest j
        win_prob_params (tuple): tuple of two N-by-M matrices (Q, R) storing winning probability parameters
            - For player i participating in contest j with effort e, the production is f_ij = Q[i, j]*e^R[i, j]
            - Player i wins with probability f_ij / (f_ij + F_j), where F_j is the sum of production from other players who currently participate in contest j
        cost_params (tuple): tuple of two N-by-M matrices (C, D) storing cost parameters
            - For player i participating in contest j with effort e, the cost is C[i, j]*e^D[i, j]

    Returns: np.ndarray
        Array of priority scores for the players.
    '''

    # Placeholder (replace with your actual implementation)
    scores = ...

    return scores
\end{lstlisting}
\end{listing*}

\clearpage


\begin{figure*}[h]
    \footnotesize
    \centering
    \Description{Code box showing the Python source of the LLMScore-designed scoring algorithm that achieved the highest overall effort during testing, together with a plain-language summary of how it scores each contestant.}
    \begin{promptbox}{0.9\textwidth}{LLM-Generated Prompt Strategy `MN1'}
        Propose an algorithm that balances player incentives by adjusting the utility function of each player based on their rank within the contest, ensuring it promotes fairness and maximizes social welfare while maintaining efficiency.
    \end{promptbox}
    \medskip
    \begin{promptbox}{0.9\textwidth}{Algorithm Description (generated using the above `MN1' prompt strategy)}
        The new algorithm adjusts the utility function of each player based on their rank within the contest to improve fairness and maximize social welfare, balancing higher valuations with the players' relative positions.
    \end{promptbox}
    \medskip
\begin{lstlisting}[language=Python]
import numpy as np

def priority(valuations, win_prob_params, cost_params):
    '''
    Args:
        valuations (np.ndarray): N-by-M matrix where entry at (i, j) denotes valuation of player i for winning contest j
        win_prob_params (tuple): tuple of two N-by-M matrices (Q, R) storing winning probability parameters
            - For player i participating in contest j with effort e, the production is f_ij = Q[i, j]*e^R[i, j]
            - Player i wins with probability f_ij / (f_ij + F_j), where F_j is the sum of production from other players who currently participate in contest j
        cost_params (tuple): tuple of two N-by-M matrices (C, D) storing cost parameters
            - For player i participating in contest j with effort e, the cost is C[i, j]*e^D[i, j]

    Returns: np.ndarray
        Array of priority scores for the players.
    '''

    N, M = valuations.shape
    Q, R = win_prob_params
    C, D = cost_params
    scores = np.zeros(N)

    for j in range(M):
        for i in range(N):
            effort = (N - i) / (N + 1) ** 2
            f_ij = Q[i, j] * (effort ** R[i, j])
            F_j = sum(Q[k, j] * ((N - k) / (N + 1) ** 2) ** R[k, j] for k in range(N) if k != i)

            winning_prob = f_ij / (f_ij + F_j + 1e-6)
            cost = C[i, j] * (effort ** D[i, j])
            adjusted_valuation = valuations[i, j] * (N - i + 1) / N  # Adjust valuation based on player rank
            
            rank_bonus = (N - i) / N  # Introduce a rank bonus to promote fairness
            priority_score = ((adjusted_valuation * np.log(winning_prob + 1e-6)) - 
                              (np.log(cost + 1e-6))) * rank_bonus

            scores[i] += priority_score

    return scores
\end{lstlisting}
    \caption{High-level description of an LLM-generated scoring algorithm with overall highest OE during testing and its associated code. The score for contestant $i\in N$ can be formally expressed as $\sigma_i=\phi_i\sum^m_{j=1}[\hat{v}_{ij}\ln(p_{ij}(e_{ij}))-\ln(c_{ij}(e_{ij}))]$, where ``adjusted valuation'' $\hat{v}_{ij}=v_{ij}\cdot\frac{n-i+1}{n}$ and ``rank bonus'' $\phi_i=\frac{n-i}{n}$. The effort $e_{ij}$ for defining $p_{ij}(e_{ij})=\frac{f_{ij}(e_{ij})}{\sum_{l\in N}f_{lj}(e_{lj})}$ and $c_{ij}(e_{ij})$ is $e_{ij}=\frac{n-i}{(n+1)^2}$ for any $j\in M$.}
    \label{fig:learned-alg1}
\end{figure*}

\begin{figure*}[h]
    \footnotesize
    \centering
    \Description{Code box showing the Python source of the LLMScore-designed scoring algorithm that achieved the lowest maximum regret during testing, together with a plain-language summary of how it scores each contestant.}
    \begin{promptbox}{0.8\textwidth}{Algorithm Description (generated using the initial M1 prompt strategy)}
        The new algorithm simplifies the priority score calculation by considering only the maximum production among current participants in a contest and applying a logarithmic scaling to both winning probability and cost to reduce fluctuations in scores due to the number of participants.
    \end{promptbox}
    \medskip
\begin{lstlisting}[language=Python]
import numpy as np

def priority(valuations, win_prob_params, cost_params):
    '''
    Args:
        valuations (np.ndarray): N-by-M matrix where entry at (i, j) denotes valuation of player i for winning contest j
        win_prob_params (tuple): tuple of two N-by-M matrices (Q, R) storing winning probability parameters
            - For player i participating in contest j with effort e, the production is f_ij = Q[i, j]*e^R[i, j]
            - Player i wins with probability f_ij / (f_ij + F_j), where F_j is the sum of production from other players who currently participate in contest j
        cost_params (tuple): tuple of two N-by-M matrices (C, D) storing cost parameters
            - For player i participating in contest j with effort e, the cost is C[i, j]*e^D[i, j]

    Returns: np.ndarray
        Array of priority scores for the players.
    '''

    N, M = valuations.shape
    Q, R = win_prob_params
    C, D = cost_params
    scores = np.zeros(N)

    for i in range(N):
        for j in range(M):
            effort = (i + 1) / (N + 1) ** 2
            f_ij = Q[i, j] * (effort ** R[i, j])
            F_j_max = 0
            
            for k in range(N):
                if i != k:
                    f_kj = Q[k, j] * ((k + 1) / (N + 1) ** 2) ** R[k, j]
                    F_j_max = max(F_j_max, f_kj)  # Use max for production effect
            
            winning_prob = f_ij / (f_ij + F_j_max + 1e-6)
            cost = C[i, j] * (effort ** D[i, j])
            log_competition_effect = np.log(F_j_max + 1)  # Logarithmic scaling to competition effect

            priority_score = (valuations[i, j] * np.log(winning_prob + 1e-6)) - (np.log(cost + 1e-6)) - log_competition_effect

            scores[i] += priority_score

    return scores
\end{lstlisting}
    \caption{High-level description of an LLM-generated scoring algorithm with overall lowest max regret during testing and its associated code. The score for contestant $i\in N$ can be formally expressed as $\sigma_i=\sum^m_{j=1}[v_{ij}\ln(p_{ij}(e_{ij}))-\ln(c_{ij}(e_{ij}))-\gamma_{ij}]$, where ``log competition effort'' $\gamma_{ij}=\ln(\hat{F}_{ij}+1)$ and $\hat{F}_{ij}=\max_{l\in N:\,l\neq i}f_{lj}(e_{lj})$. The effort $e_{ij}$ for defining $p_{ij}(e_{ij})=\frac{f_{ij}(e_{ij})}{\sum_{l\in N}f_{lj}(e_{lj})}$, $c_{ij}(e_{ij})$, and $\hat{F}_{ij}$ is $e_{ij}=\frac{i+1}{(n+1)^2}$ for any $j\in M$.}
    \label{fig:learned-alg3}
\end{figure*}


\begin{figure*}[h]
    \centering
    \begin{subfigure}[h]{0.48\textwidth}
        \centering
        \footnotesize
        \Description{Six grouped bar charts, one per SSTC 3 instance size (n10m3, n15m3, n30m3, n60m3, n30m5, n60m5), each showing the percentage improvement in overall effort (solid bars) and the reduction in maximum regret (striped bars) achieved by GRAF with LLMScore-designed scores relative to the baselines.}
        \includegraphics[width=0.99\textwidth]{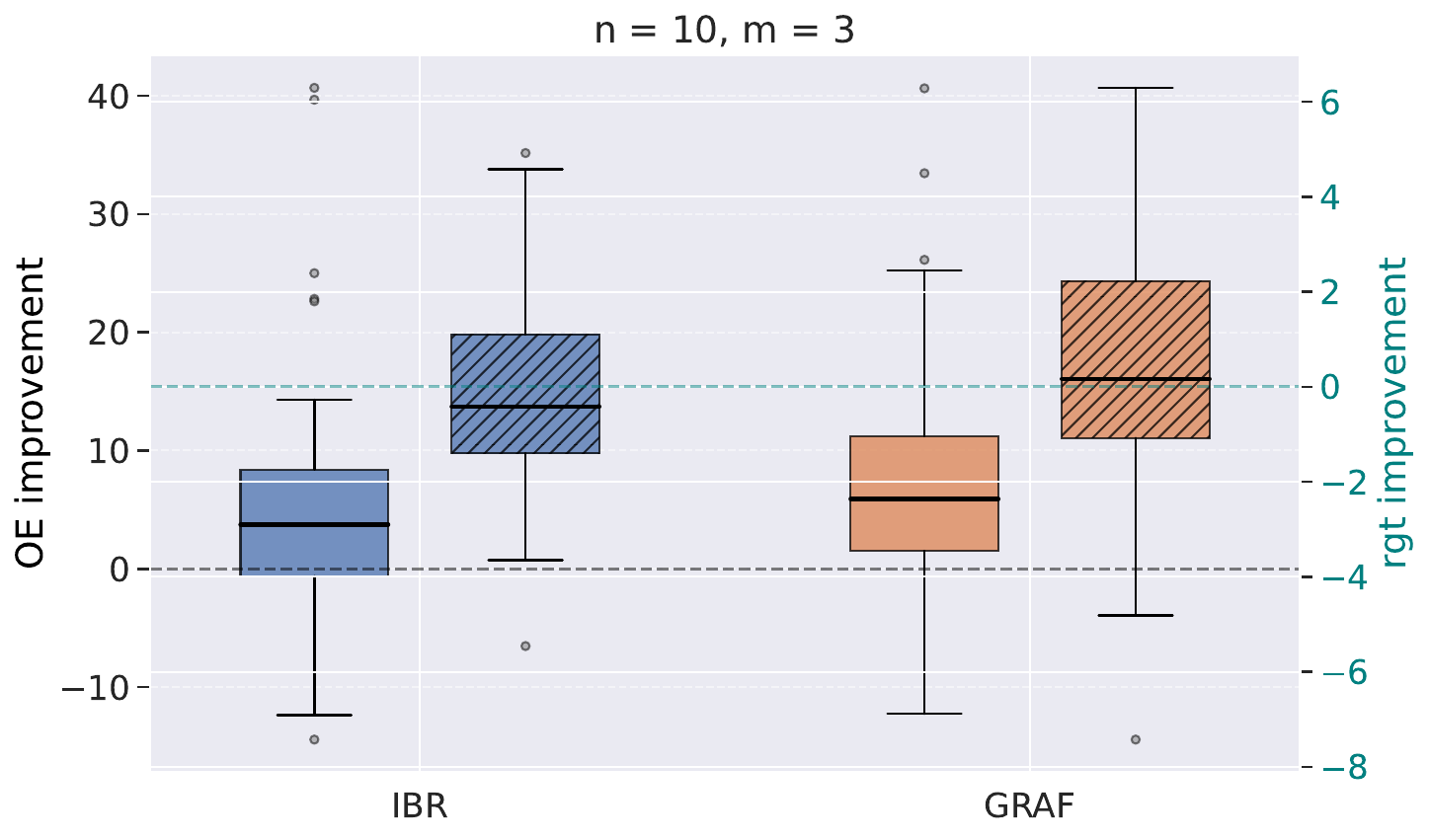}
    \end{subfigure}%
    \begin{subfigure}[h]{0.48\textwidth}
        \centering
        \footnotesize
        \includegraphics[width=0.99\textwidth]{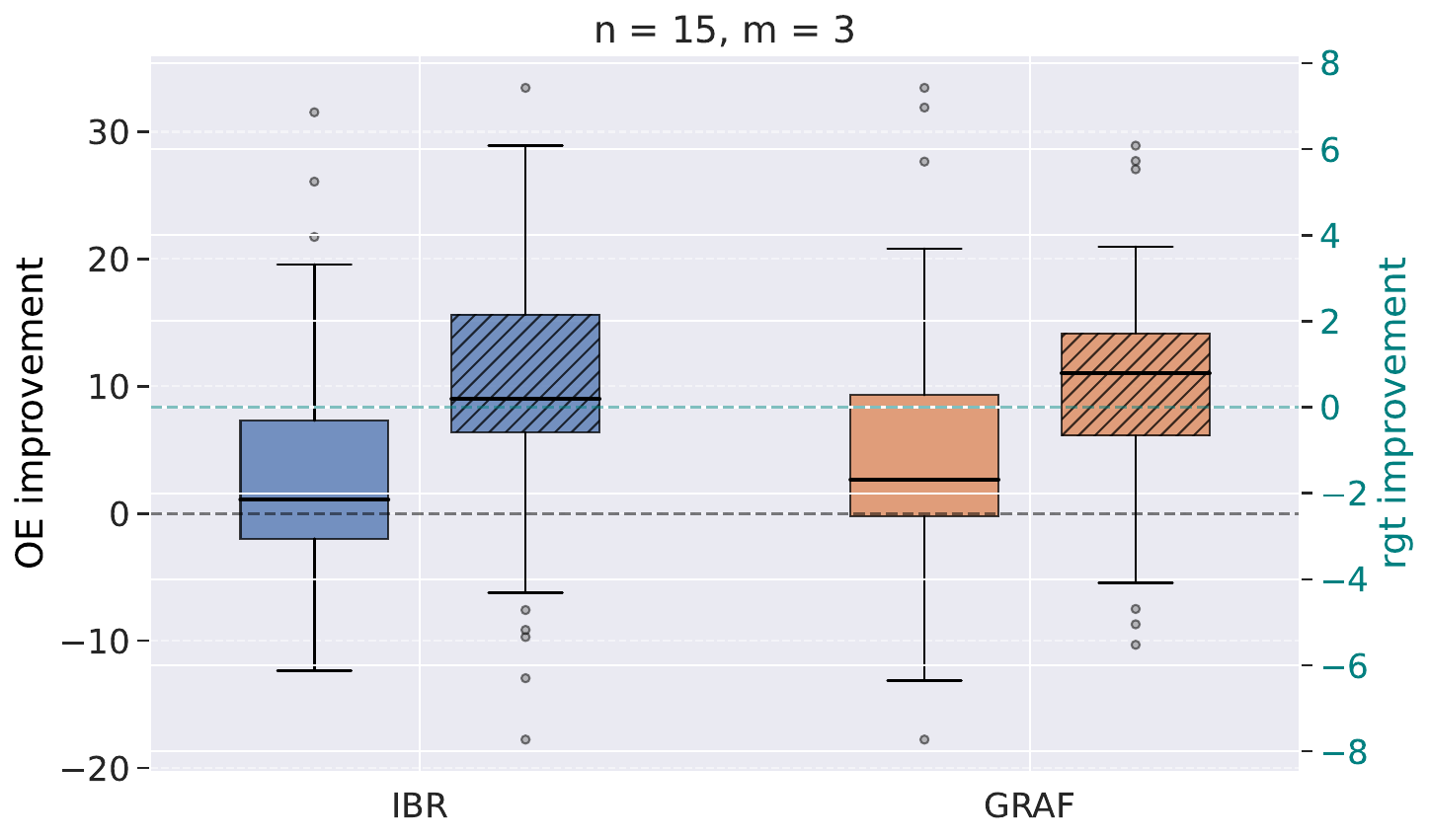}
    \end{subfigure}\\
    \begin{subfigure}[h]{0.48\textwidth}
        \centering
        \footnotesize
        \includegraphics[width=0.99\textwidth]{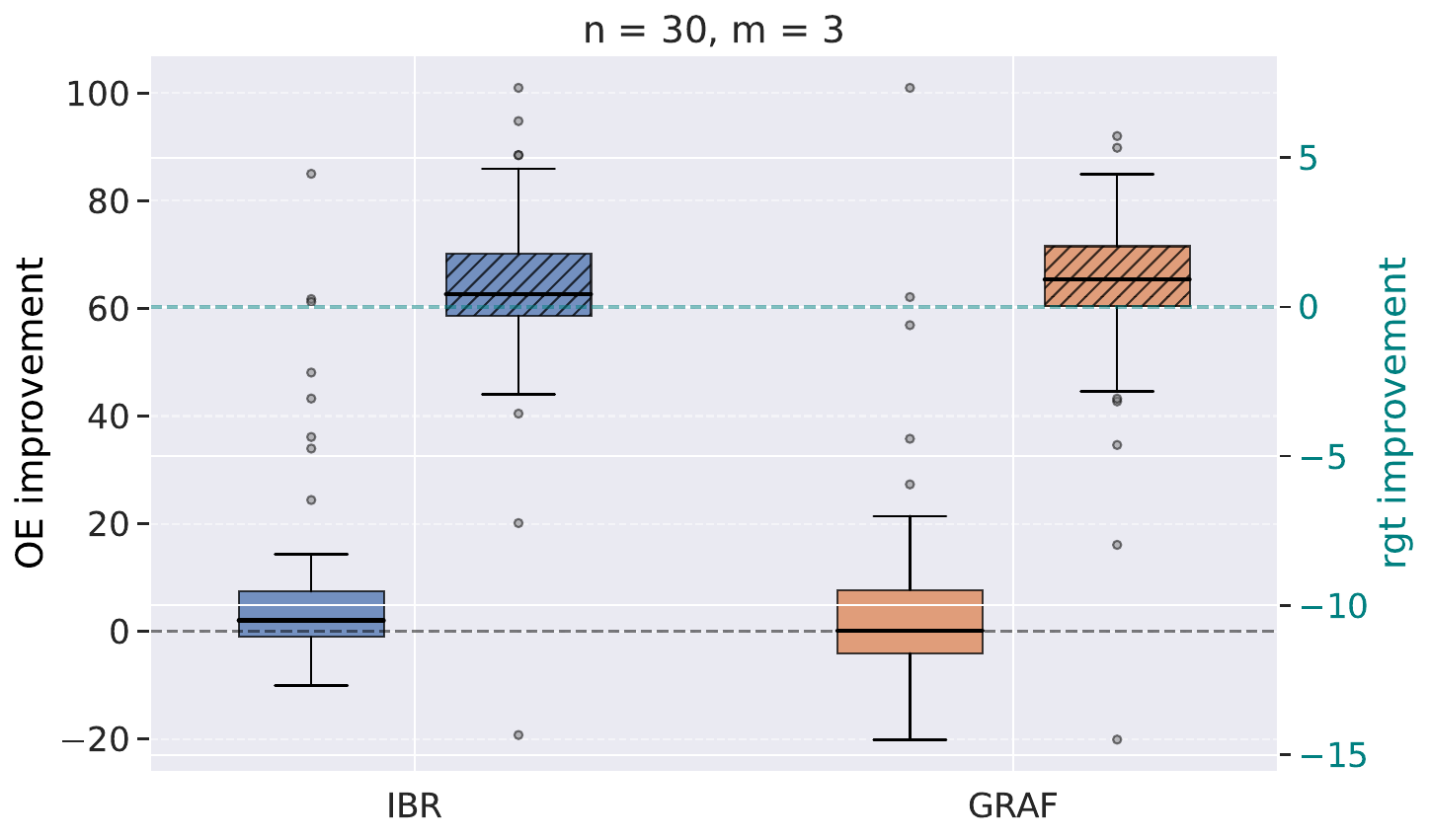}
    \end{subfigure}%
    \begin{subfigure}[h]{0.48\textwidth}
        \centering
        \footnotesize
        \includegraphics[width=0.99\textwidth]{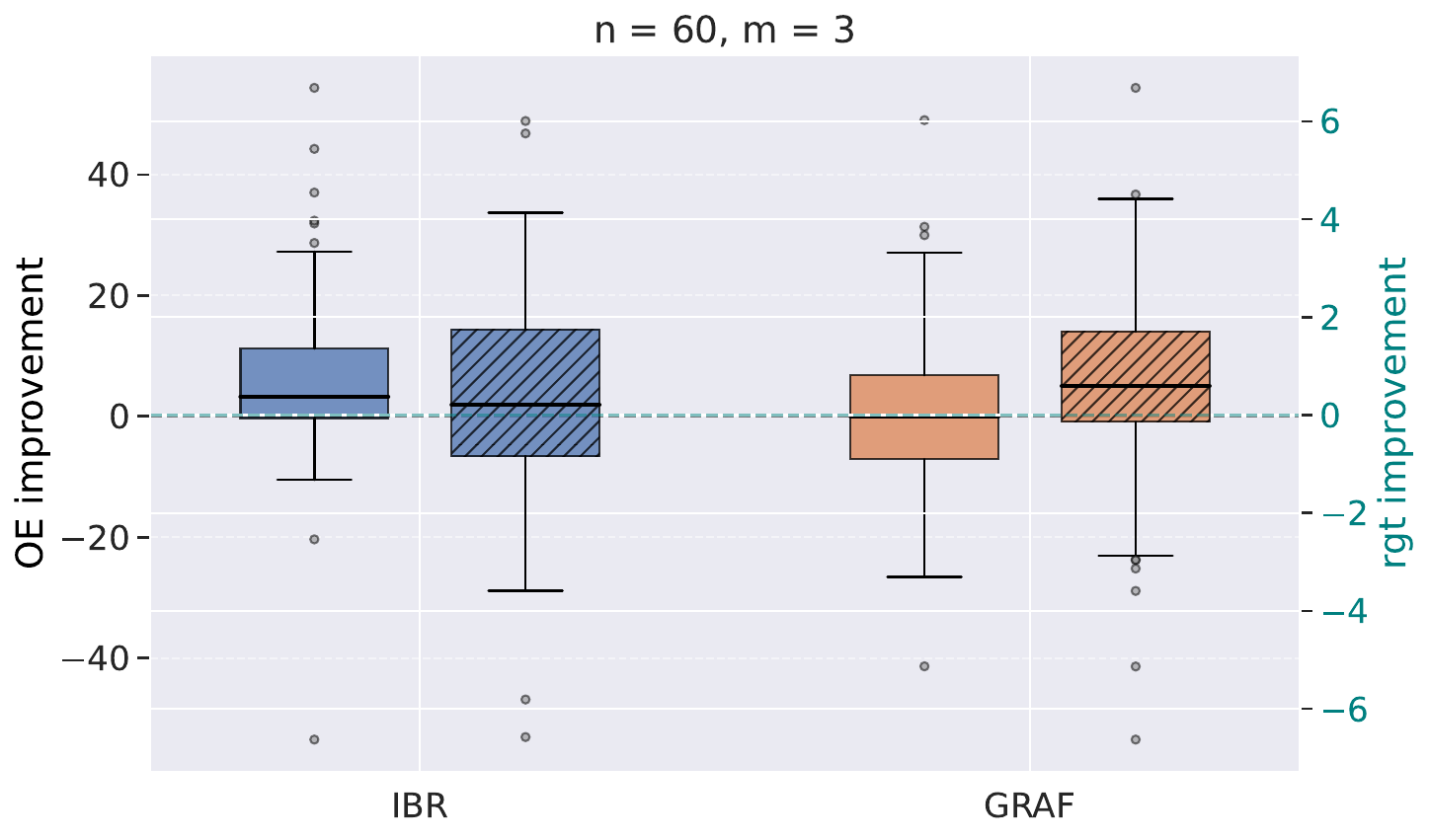}
    \end{subfigure}\\
    \begin{subfigure}[h]{0.48\textwidth}
        \centering
        \footnotesize
        \includegraphics[width=0.99\textwidth]{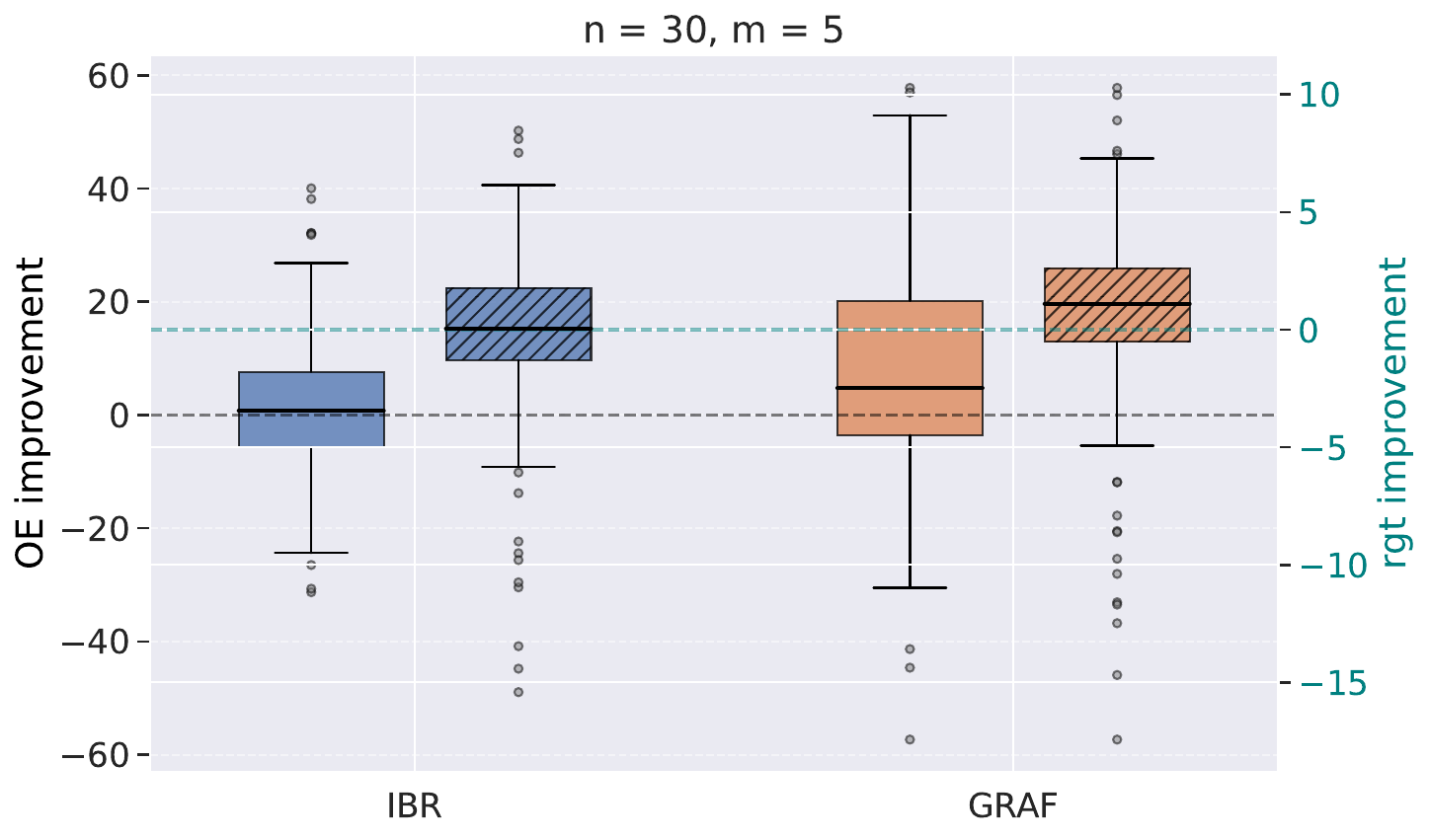}
    \end{subfigure}%
    \begin{subfigure}[h]{0.48\textwidth}
        \centering
        \footnotesize
        \includegraphics[width=0.99\textwidth]{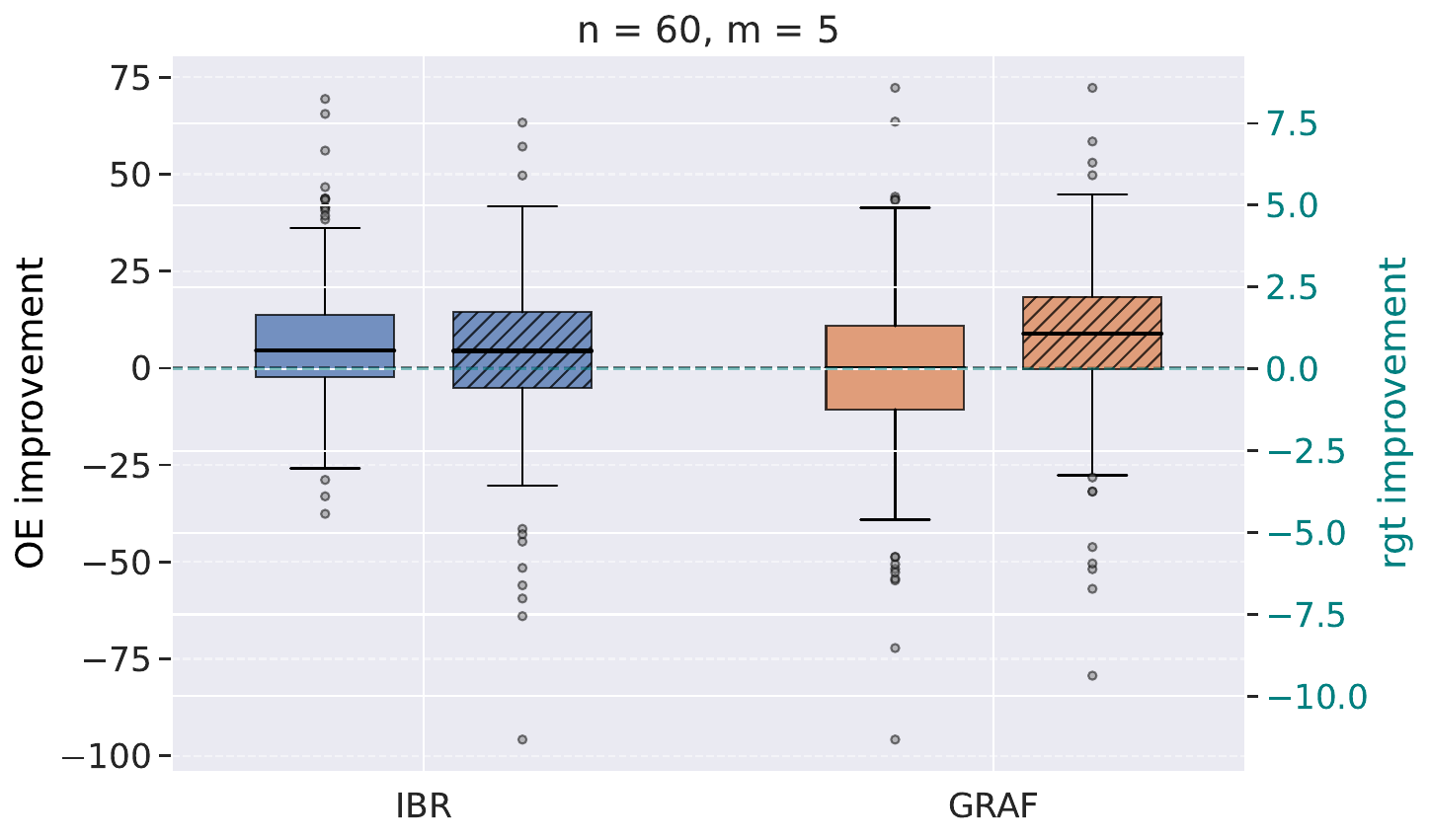}
    \end{subfigure}
    \caption{Improvement in OE (solid) and max regret (striped) in SSTC 3. (Similar results for SSTC 4 in Figure \ref{fig:box-sstc_general}.)}
    \label{fig:box-ssstc_general}
\end{figure*}

\begin{figure*}[t]
    \centering
    \begin{subfigure}[h]{0.32\textwidth}
        \centering
        \footnotesize
        \Description{Six scatter plots, one per SSTC 3 instance size, plotting median overall effort on the horizontal axis against negative median maximum regret on the vertical axis for IBR (squares), GRAF with handcrafted scores (triangles), and the top-ten GRAF with LLMScore scores (circles); the Pareto front is drawn in red, with the LLMScore points clustering toward the high-effort, low-regret upper-right region.}
        \includegraphics[width=0.99\textwidth]{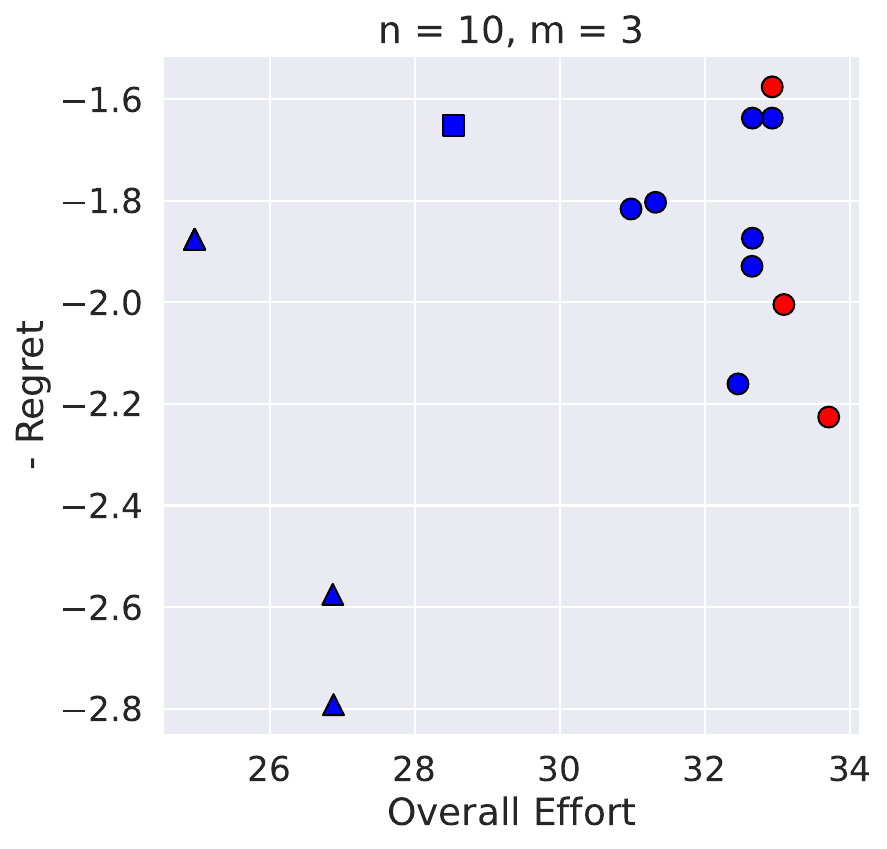}
    \end{subfigure}%
    \begin{subfigure}[h]{0.32\textwidth}
        \centering
        \footnotesize
        \includegraphics[width=0.99\textwidth]{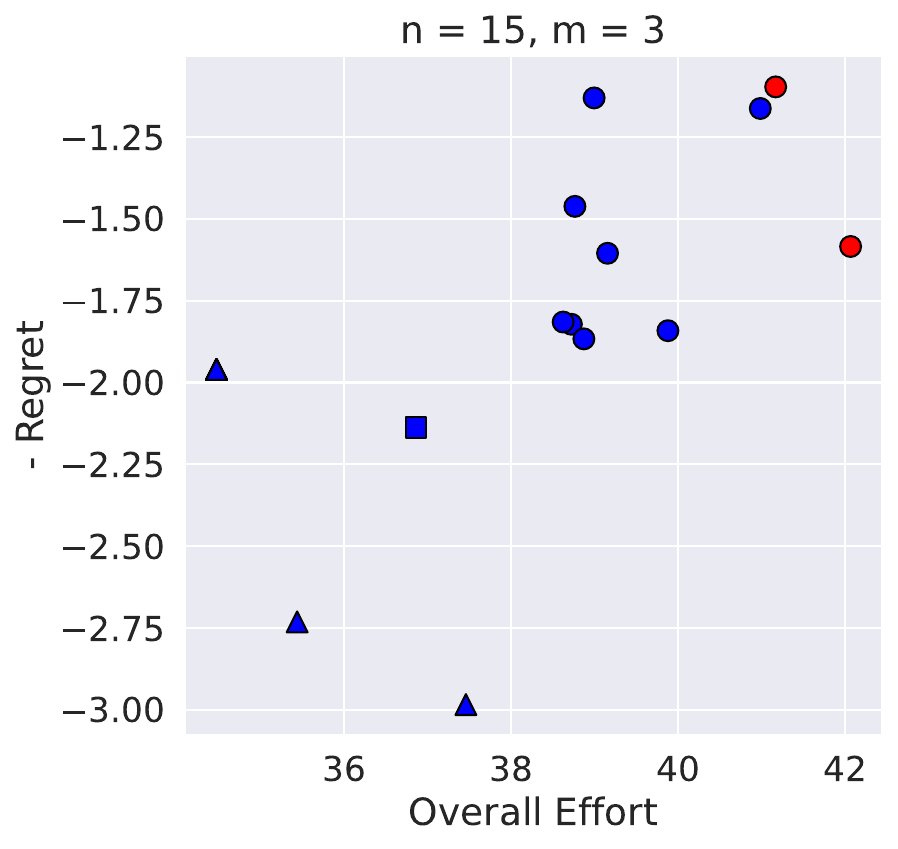}
    \end{subfigure}%
    \begin{subfigure}[h]{0.32\textwidth}
        \centering
        \footnotesize
        \includegraphics[width=0.99\textwidth]{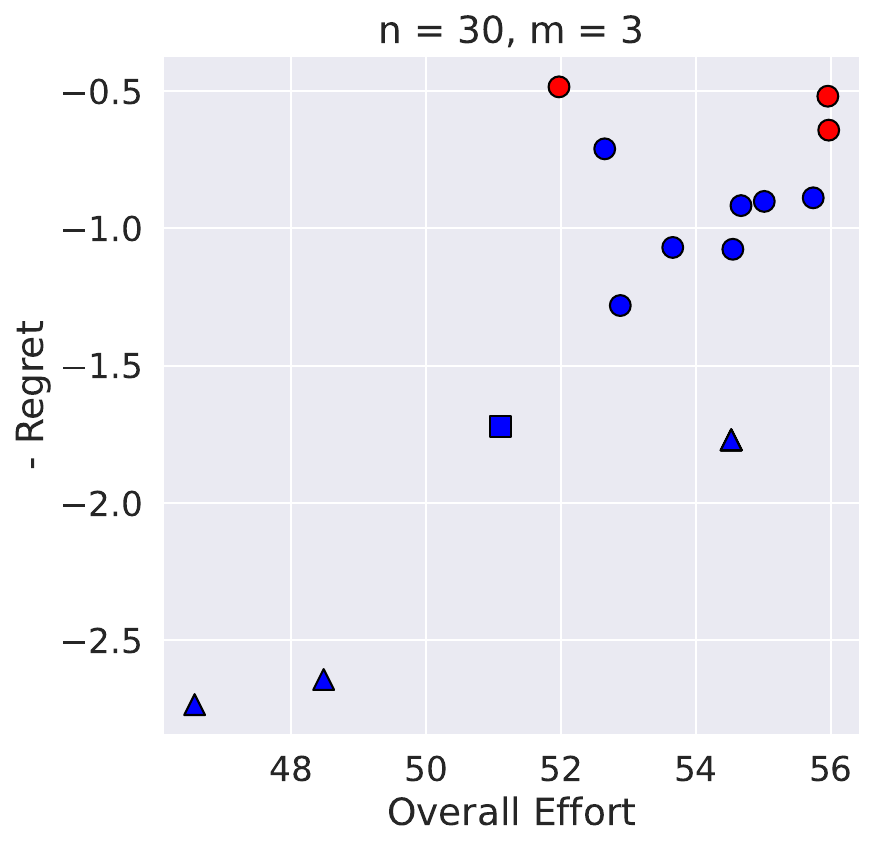}
    \end{subfigure}\\
    \begin{subfigure}[h]{0.32\textwidth}
        \centering
        \footnotesize
        \includegraphics[width=0.99\textwidth]{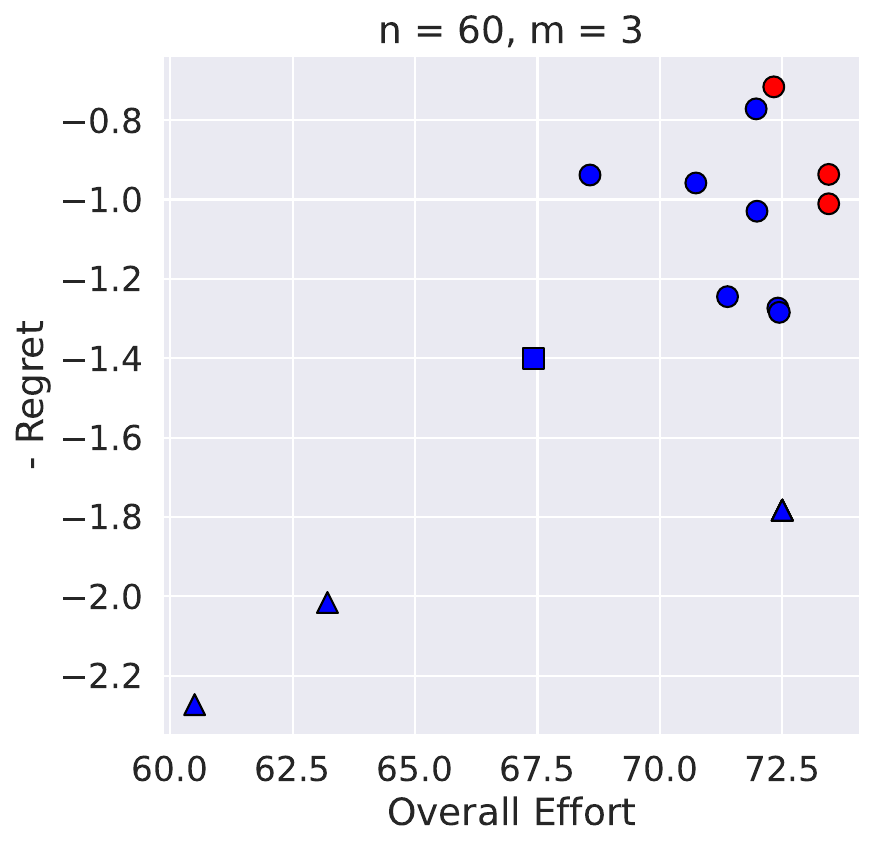}
    \end{subfigure}%
    \begin{subfigure}[h]{0.32\textwidth}
        \centering
        \footnotesize
        \includegraphics[width=0.99\textwidth]{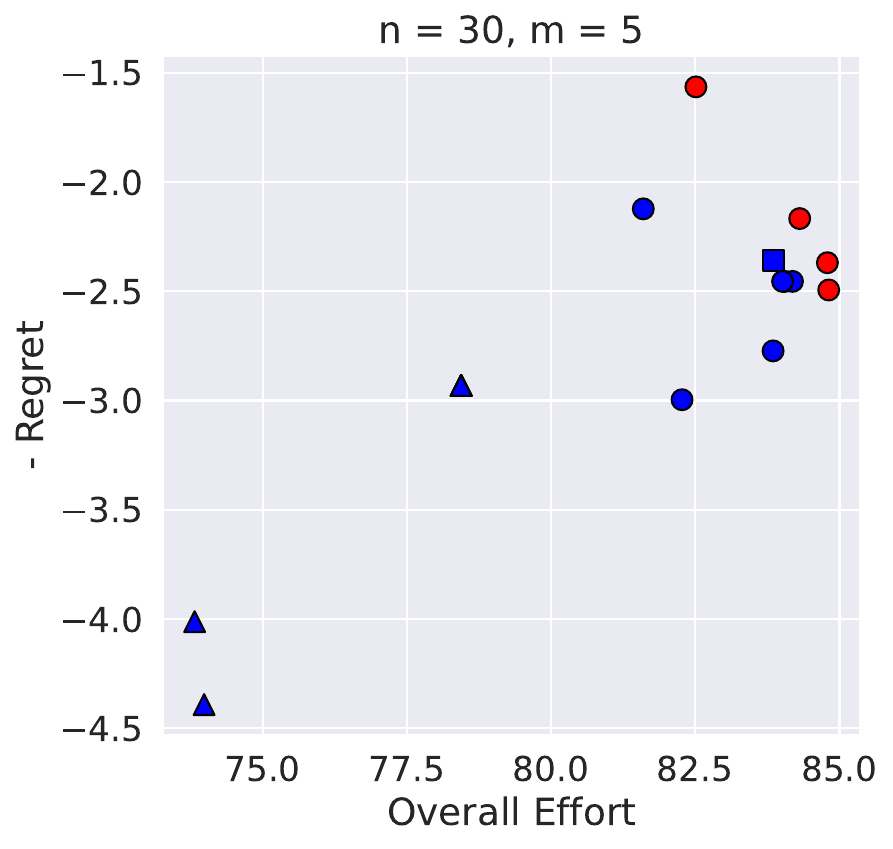}
    \end{subfigure}%
    \begin{subfigure}[h]{0.32\textwidth}
        \centering
        \footnotesize
        \includegraphics[width=0.99\textwidth]{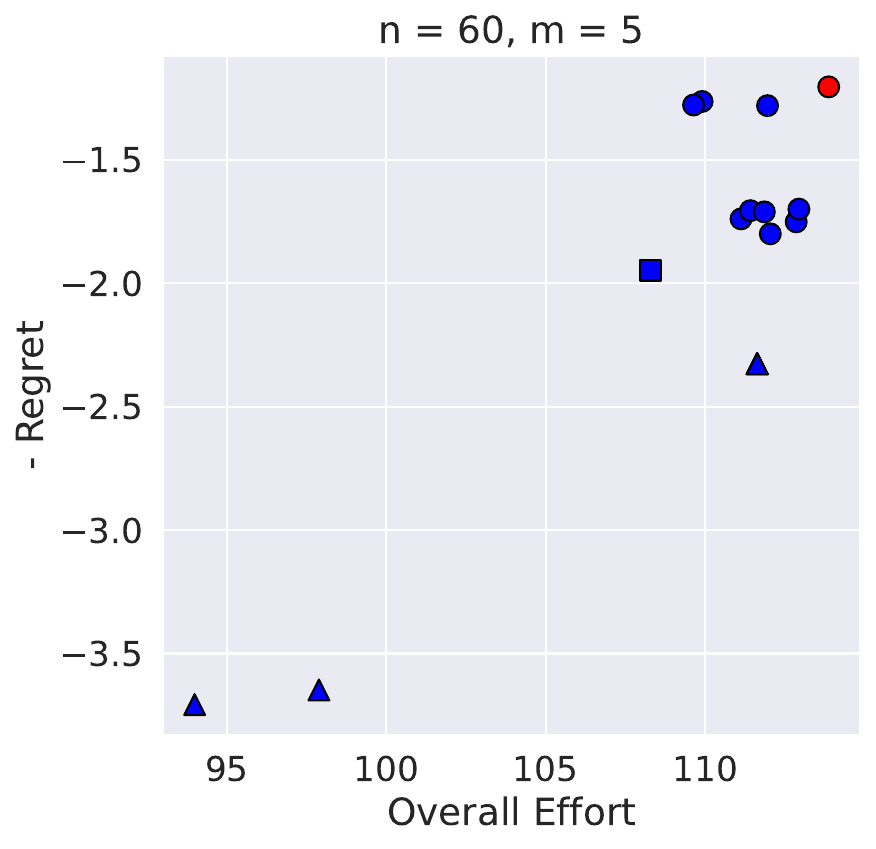}
    \end{subfigure}
    \caption{Median OE vs. (minus) median max regret in SSTC 3. The Pareto front is highlighted red. Legends: $\blacksquare$---IBR; $\blacktriangle$---GRAF[handcrafted]; \scalebox{1.5}{$\bullet$}---GRAF[LLMScore] (top 10). (Similar results for SSTC 4 in Figure \ref{fig:tradeoff-sstc_general}.)}
    \label{fig:tradeoff-ssstc_general}
\end{figure*}

\begin{figure*}[h]
    \centering
    \begin{subfigure}[h]{0.48\textwidth}
        \centering
        \footnotesize
        \Description{Five grouped bar charts, one per SSTC 4 instance size, each showing the percentage improvement in overall effort (solid bars) and the reduction in maximum regret (striped bars) achieved by GRAF with LLMScore-designed scores relative to the baselines.}
        \includegraphics[width=0.99\textwidth]{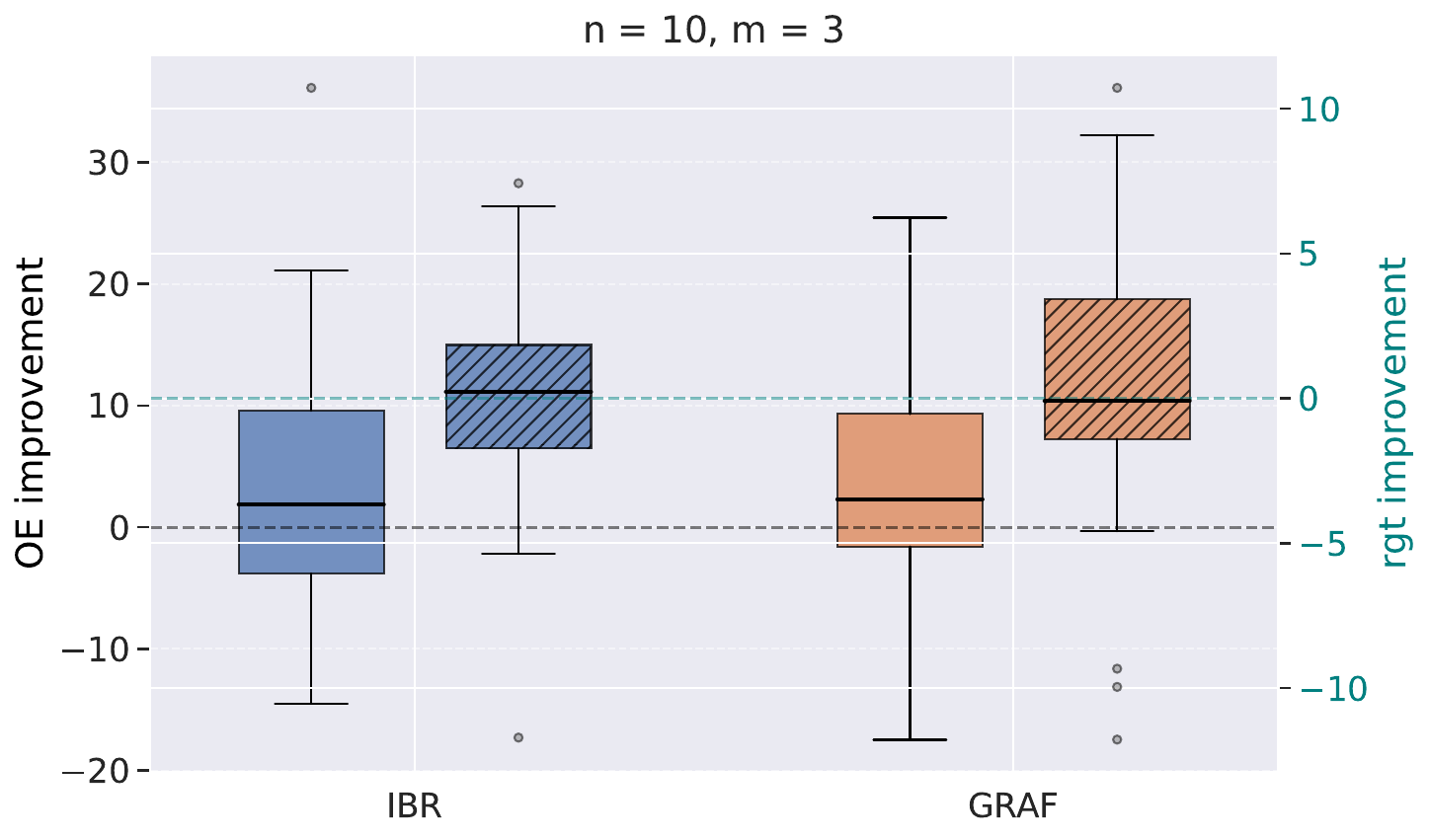}
    \end{subfigure}%
    \begin{subfigure}[h]{0.48\textwidth}
        \centering
        \footnotesize
        \includegraphics[width=0.99\textwidth]{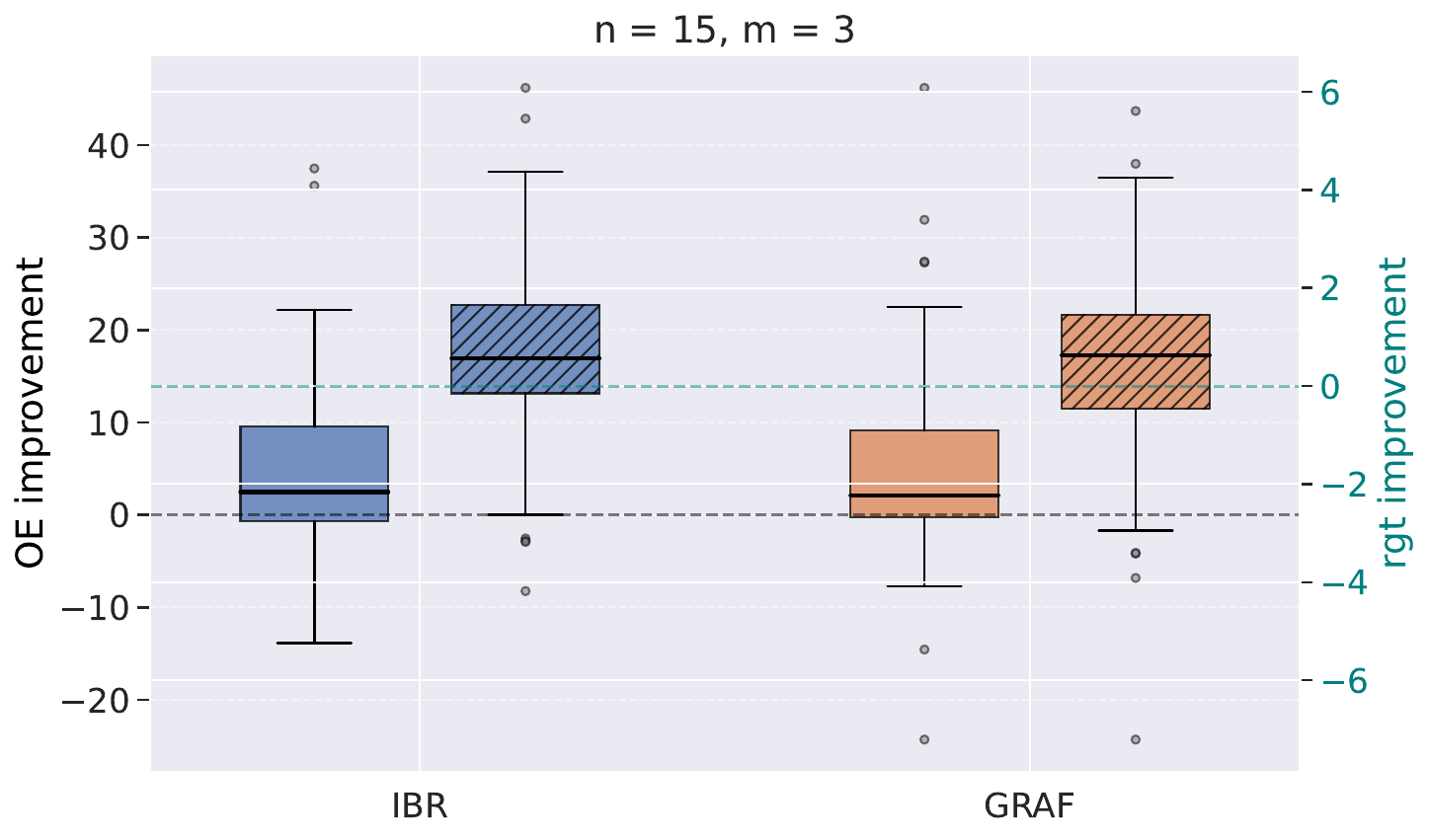}
    \end{subfigure}\\
    \begin{subfigure}[h]{0.48\textwidth}
        \centering
        \footnotesize
        \includegraphics[width=0.99\textwidth]{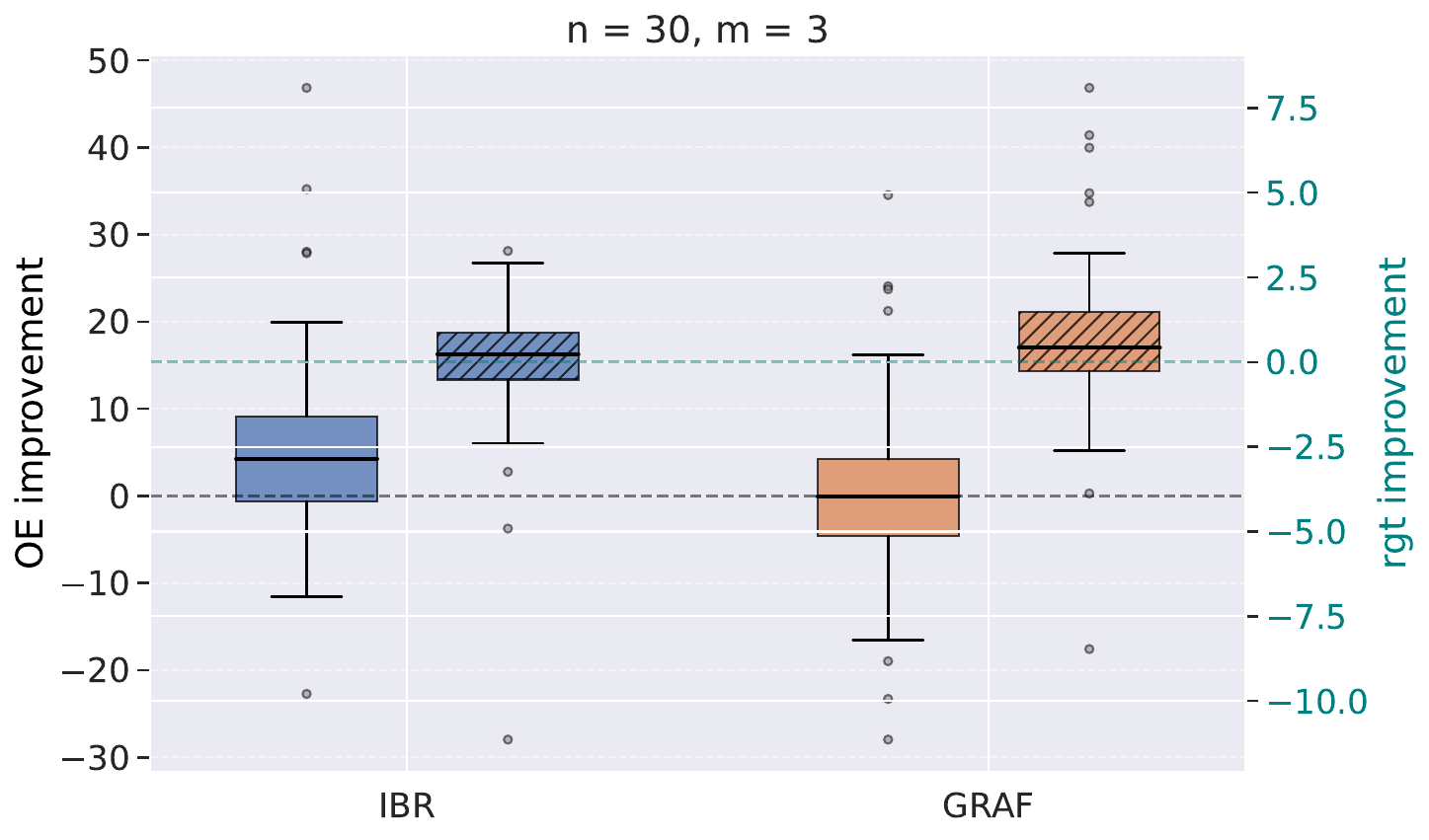}
    \end{subfigure}%
    \begin{subfigure}[h]{0.48\textwidth}
        \centering
        \footnotesize
        \includegraphics[width=0.99\textwidth]{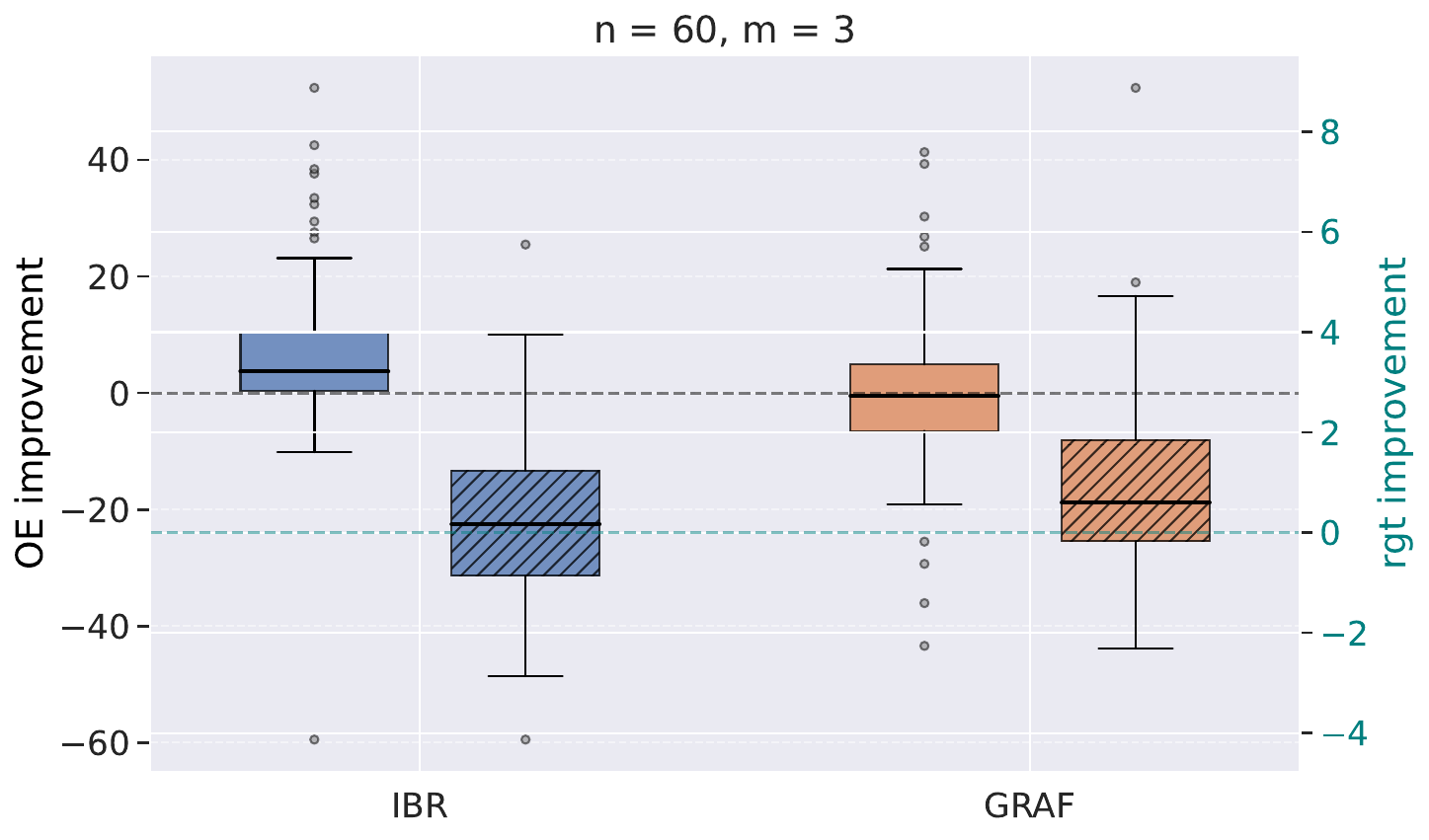}
    \end{subfigure}\\
    \begin{subfigure}[h]{0.48\textwidth}
        \centering
        \footnotesize
        \includegraphics[width=0.99\textwidth]{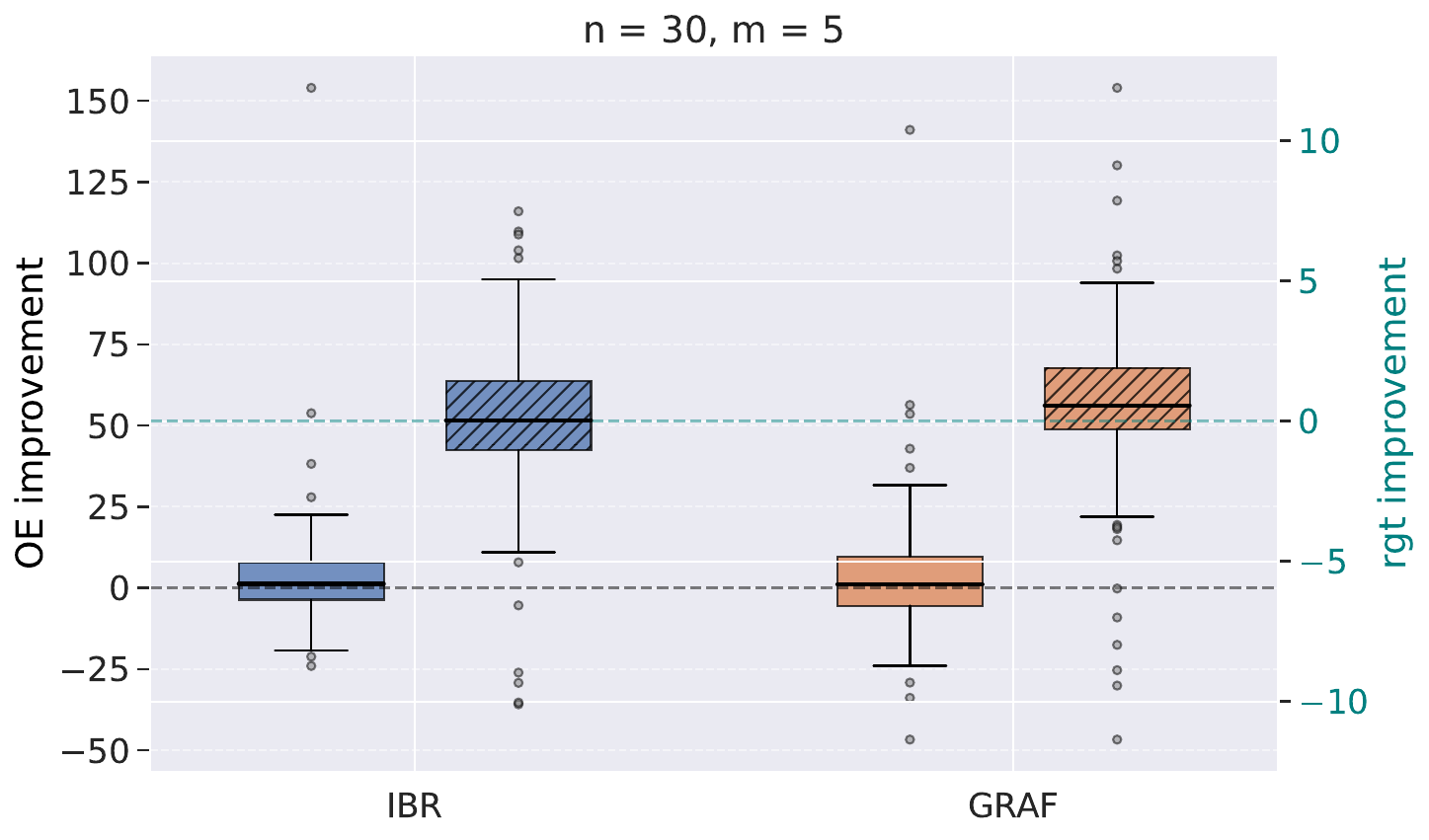}
    \end{subfigure}%
    \caption{Improvement in OE (solid) and max regret (striped) in SSTC 4.}
    \label{fig:box-sstc_general}
\end{figure*}

\begin{figure*}[h]
    \centering
    \begin{subfigure}[h]{0.32\textwidth}
        \centering
        \footnotesize
        \Description{Five scatter plots, one per SSTC 4 instance size, plotting median overall effort on the horizontal axis against negative median maximum regret on the vertical axis for IBR (squares), GRAF with handcrafted scores (triangles), and the top-ten GRAF with LLMScore scores (circles); the Pareto front is drawn in red.}
        \includegraphics[width=0.99\textwidth]{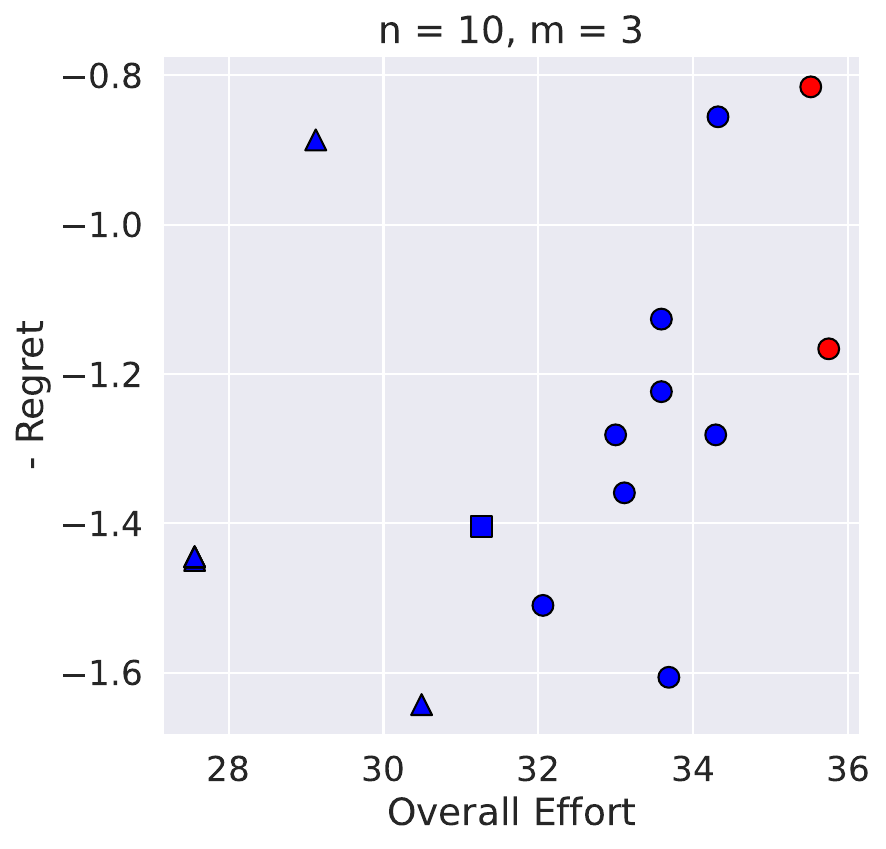}
    \end{subfigure}%
    \begin{subfigure}[h]{0.32\textwidth}
        \centering
        \footnotesize
        \includegraphics[width=0.99\textwidth]{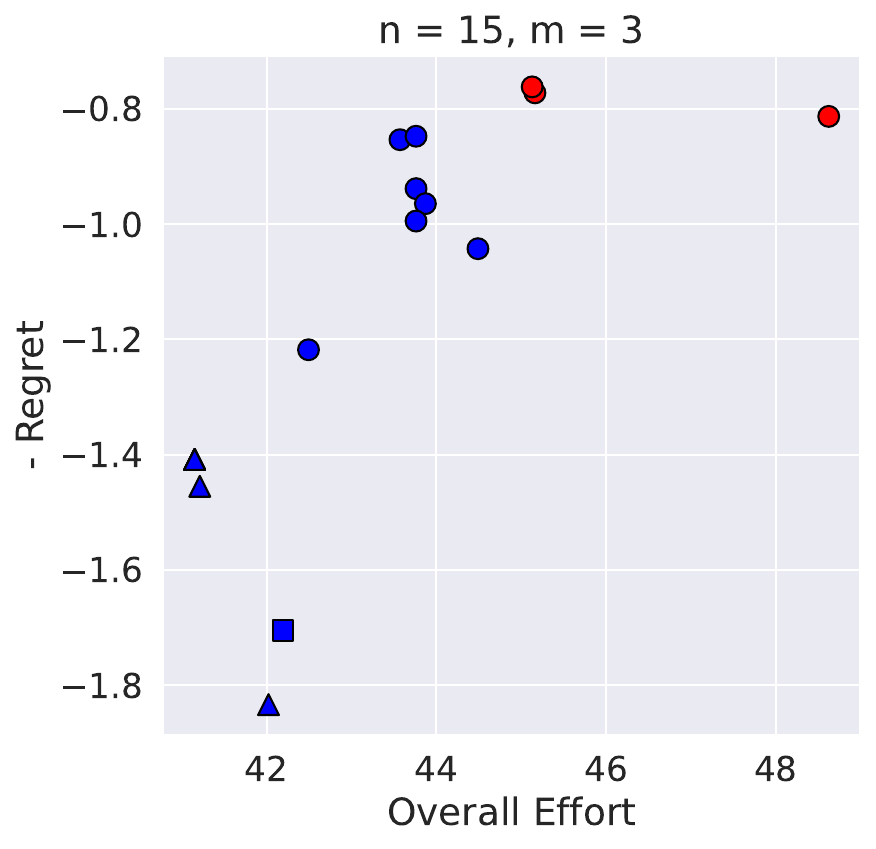}
    \end{subfigure}%
    \begin{subfigure}[h]{0.32\textwidth}
        \centering
        \footnotesize
        \includegraphics[width=0.99\textwidth]{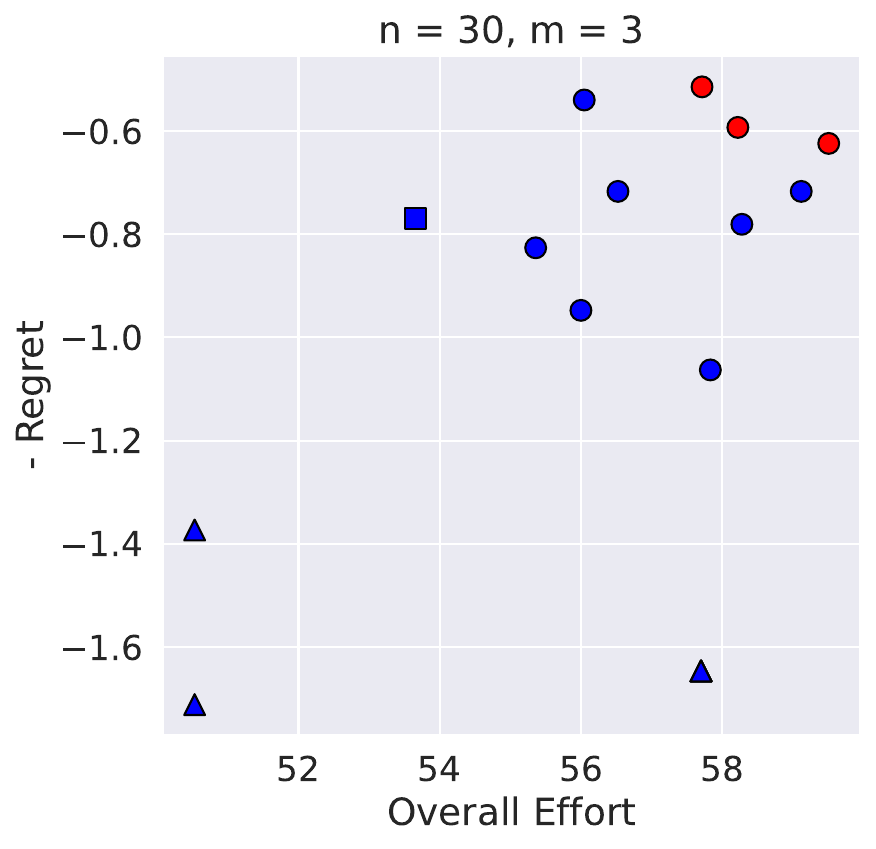}
    \end{subfigure}\\
    \begin{subfigure}[h]{0.32\textwidth}
        \centering
        \footnotesize
        \includegraphics[width=0.99\textwidth]{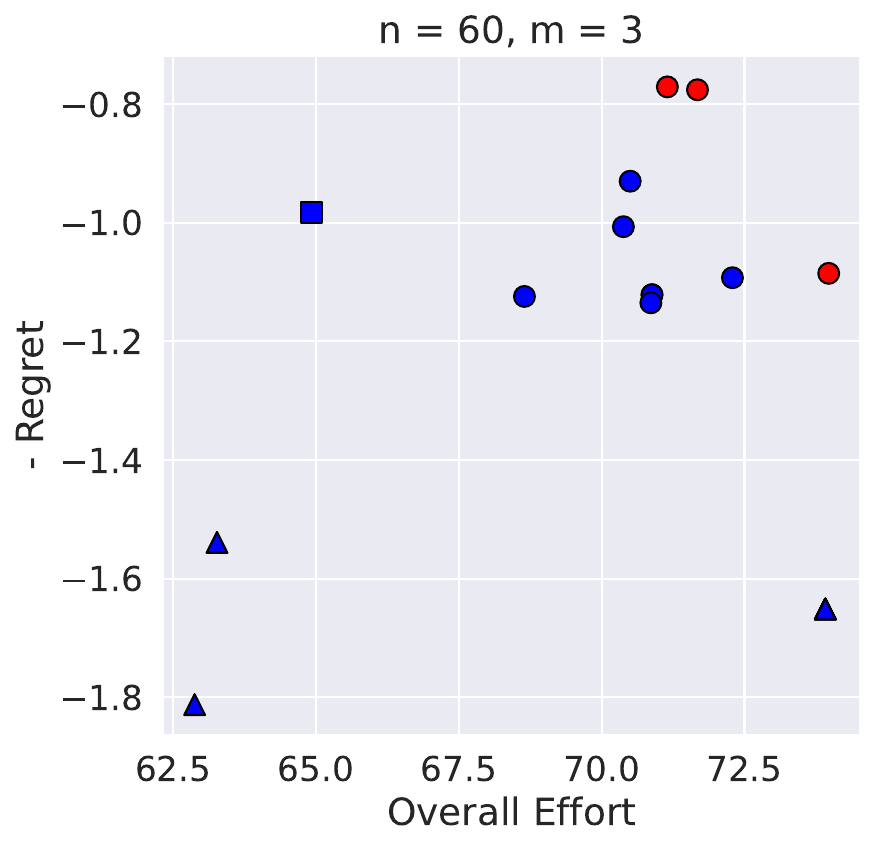}
    \end{subfigure}%
    \begin{subfigure}[h]{0.32\textwidth}
        \centering
        \footnotesize
        \includegraphics[width=0.99\textwidth]{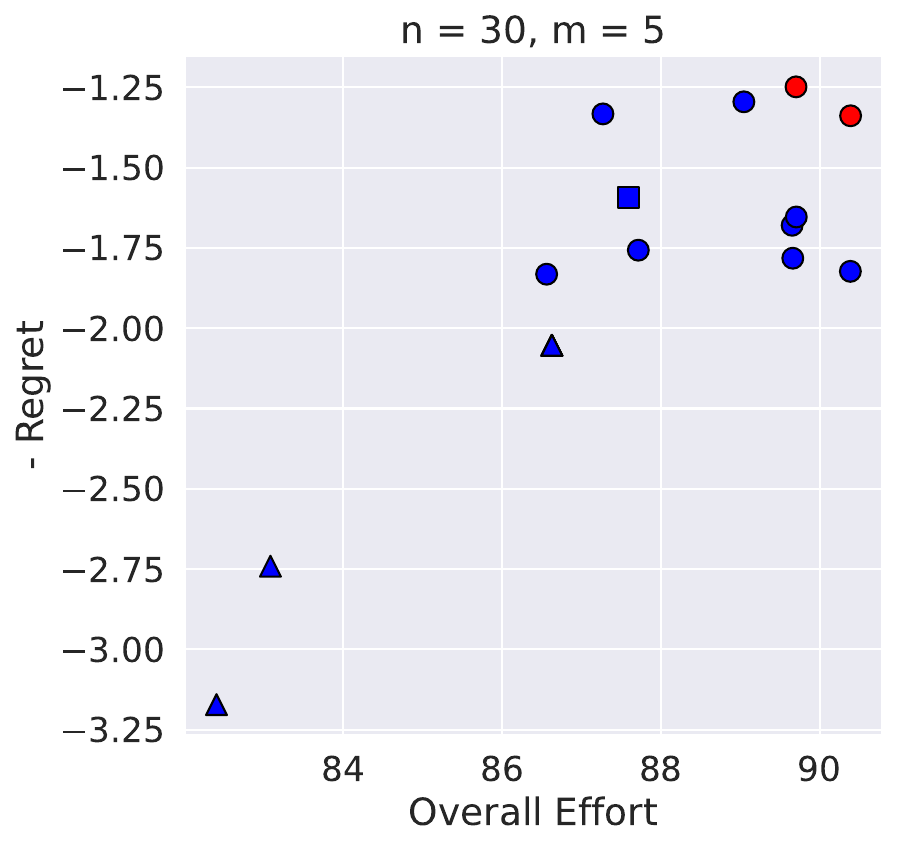}
    \end{subfigure}%
    \caption{Median OE vs. (minus) median max regret in SSTC 4. The Pareto front is highlighted red. Legends: $\blacksquare$---IBR; $\blacktriangle$---GRAF[handcrafted]; \scalebox{1.5}{$\bullet$}---GRAF[LLMScore] (top 10).}
    \label{fig:tradeoff-sstc_general}
\end{figure*}

\clearpage
\onecolumn

\section{Supplementary Material}

\subsection{Missing Proofs}\label{app:proofs}

\paragraph{Proof for Theorem 2.}
\emph{For any game of special SSTC, there is a PSNE. 
Moreover, Algorithm 2 outputs one in polynomial time.}

\begin{proof}
The key to showing that Algorithm 2 returns 
a PSNE is to argue that (1) the partial self-selection outcome $\mathbf{a}_{-i}$ is a PSNE for 
the game played between the contestants up until $i$ 
and (2) the self-selection outcome $(a_i,\mathbf{a}_{-i})$ is still a PSNE 
after setting a best-response of $i$ to be one.
We show this by induction. 

In the base step $i=1$, $i$ can select any contest with zero participants 
and obtain a utility of $v_1$, which is the highest $i$ can obtain, 
by exerting zero effort. Thus, $\mathbf{a}$ is a PSNE. 
This holds for $1 \le i \le m$. 

In the inductive step, we assume that $\mathbf{a}_{-i}$ 
is a PSNE for the game played between 
the contestants up until $i$. We want to show that $(a_i, \mathbf{a}_{-i})$ 
is also a PSNE where $a_i$ is constructed as in Algorithm 1. 
Let $j$ be the contest such that $a_{ij} = 1$ (Step 5). 
If $u_i(a_i, \mathbf{a}_{-i}) = 0$, then $i$ exerts zero effort 
at the PSNE in any contest, and we are done by the inductive assumption.
In the following, we consider the case where 
$u_i(a_i, \mathbf{a}_{-i}) > 0$ and all contestants are active. 
In this case, we need to show that (a) each contestant in $j$ would not deviate 
to a different contest (with the addition of $i$) and (b) the contestants that are 
not in $j$ would not deviate to contest $j$. 

\begin{claim}
No contestants in contest $j$ would deviate to a different contest 
after $i$ joining the contest $j$. 
\end{claim}
\begin{proof}
We let $C(j) = \{ l \in \{1, \ldots, i\} \; | \; a_{lj} = 1 \}$
be the set of contestants that selected $j$.  We let $p_{j\cup\{l\}}$ and $\Gamma_{p_{j\cup\{l\}}}$ 
to be the equilibrium parameters in Theorem 1 
restricting to $C(j)\cup\{l\}$ for any contestant $l$. 
Since $j \in \mathcal{BR}_i(\mathbf{a}_{-i}) = \argmax_{a_i \in A_i} u_i(a_i,\mathbf{a}_{-i})$, 
\begin{align*}
u_i(a_i, \mathbf{a}_{-i}) &=  v_i\left(1- \frac{p_{j}-1}{v_i\Gamma_{p_{j}}}\right) - \frac{p_{j}-1}{\Gamma_{p_{j}}}\left(1- \frac{p_{j}-1}{v_i\Gamma_{p_{j}}}\right)\\
&=v_i\left(1- \frac{p_{j}-1}{v_i\Gamma_{p_{j}}}\right)^2 \ge  
v_i\left(1- \frac{p_{j'\cup\{i\}}-1}{v_i\Gamma_{p_{j'\cup\{i\}}}}\right)^2 = u_i(a_i', \mathbf{a}_{-i}), \nonumber
\end{align*}
for any $a_i'$ such that $a_{ij'}' = 1$ and $i$ is active in $j'$. 
From the above equilibrium condition, we have
\begin{align*}
v_i\left(1- \frac{p_{j}-1}{v_i\Gamma_{p_{j}}}\right)^2 &\ge  v_i\left(1- \frac{p_{j'\cup\{i\}}-1}{v_i\Gamma_{p_{j'\cup\{i\}}}}\right)^2 \\
\Longleftrightarrow\frac{p_{j}-1}{\Gamma_{p_{j}}} &\le \frac{p_{j'\cup\{i\}}-1}{\Gamma_{p_{j'\cup\{i\}}}} \\
\Longleftrightarrow \frac{p_{j}-1}{\sum_{w\in C(j)} \frac{1}{v_w}} &\le \frac{p_{j'\cup\{i\}}-1}{\sum_{w \in C(j')} \frac{1}{v_w} + \frac{1}{v_i}}. 
\end{align*}
Consider any $l \in C(j)$, we have  
$v_l \ge v_i$ and $\frac{1}{v_l} \le \frac{1}{v_i}$ and 
\begin{align*}
\frac{p_{j}-1}{\sum_{w\in C(j)} \frac{1}{v_w}} &\le \frac{p_{j'\cup\{i\}}-1}{\sum_{w \in C(j')} \frac{1}{v_w} + \frac{1}{v_i}} 
\le \frac{p_{j'\cup\{l\}}-1}{\sum_{w \in C(j')} \frac{1}{v_w} + \frac{1}{v_l}}, 
\end{align*}
where the second inequality holds through conditions 
(i) if all contestants in $j'$ are still active after $l$ joining the contests, 
then we can simply replace $\frac{1}{v_i}$ by $\frac{1}{v_l}$ and the second inequality holds
or (ii) if there is some contestant $w$ that is not active after $l$ joining the contest $j'$, 
then the last term (minus the inactive contestants) is no less than $v_w$ and $v_w$ is no less than the second term 
since $w$ is active after $i$ joining the contest. 
Thus, for any $l \in C(j)$, we have $u_l(a_i, \mathbf{a}_{-i})$
{\small
\begin{align*}
 &= v_l\left(1- \frac{p_{j}-1}{v_l\Gamma_{p_{j}}}\right)^2 \ge v_l\left(1- \frac{p_{j'\cup\{l\}}-1}{v_l\Gamma_{p_{j'\cup\{l\}}}}\right)^2 =  u_l(a_l', \mathbf{a}_{-l}), 
\end{align*}
}
where $a_{lj'}'=1$ for any $j'$ where $i$ is active. 
Thus, no $l$ would want to deviate to any active contest of $i$. 
For contest $j'$ where $i$ is not active, 
$$
\frac{p_{j}-1}{\sum_{w\in C(j)} \frac{1}{v_w}} \le v_i \le \frac{p_{j'}-1}{\sum_{w \in C(j')} \frac{1}{v_w}} \le \frac{p_{j'}}{\sum_{w \in C(j')} \frac{1}{v_w} + \frac{1}{v_l}},
$$
where the last inequality holds (i) if $l$ and all the contestants are active in $j$ 
or (ii) if there are some contestants that are inactive after $l$ joining $j'$. 
If $l$ is not active in $j'$, then the first and second inequality follows. 
Thus, no $l$ would deviate to $i$'s inactive contests. 
\end{proof}
\begin{claim}
No contestants outside of contest $j$ would deviate 
to contest $j$ after $i$ joining the contest $j$. 
\end{claim}
\begin{proof}
First note that, for any contestant $l$ that is not in $C(j)$,
$l$ can potentially deviate if there is at least two contestants (including $l$) 
in $l$'s contest, otherwise, $l$ would achieve the highest utility $v_l$ 
and exert zero effort if $l$ is the only contestant. In addition, 
$l$ must be active since it was a best-response, otherwise 
$i$ is not active. 
Thus, below, we consider any contestant $l$ (not in $C(j)$) 
such that there are at least two contestants in $l$'s contest.
Moreover, we assume, without loss of generality, 
that contestant $i$ is still active in contest $j$ after $l$ joining $j$, 
otherwise, by the inductive assumption, $l$ would not deviate.
Also, without loss of generality, we consider $l$
to be the one with the highest value in $j'$ (i.e., if $l$ does not deviate 
then, other lower value contestants in $j'$ would not deviate). 
We denote $l$'s contest by $j'$.  
As such, we need to show that 
\begin{align}
v_l\left(1- \frac{p_{j'}-1}{v_l\Gamma_{p_{j'}}}\right)^2 &\ge  v_l\left(1- \frac{p_{j\cup\{l\}}-1}{v_l\Gamma_{p_{j\cup\{l\}}}}\right)^2 \nonumber \\
\Longleftrightarrow \frac{p_{j'}-1}{\sum_{w \in C(j') \setminus \{l\}} \frac{1}{v_w} + \frac{1}{v_l}} &\le \frac{p_{j\cup\{l\}}-1}{\sum_{w \in C(j)\setminus \{i\}} \frac{1}{v_w} + \frac{1}{v_l} + \frac{1}{v_i}}. \label{ineq:zero}
\end{align}
We present the following two inequalities to help us in our proof. 
Since we assume that $i$ is still active after $l$ joining $j$, by the equilibrium conditions, 
this implies that 
\begin{align}\label{ineq:one}
v_i \ge \frac{p_{j\cup\{l\}\setminus\{i\}} - 1}{\sum_{w \in C(j)\setminus\{i,l\}} \frac{1}{v_w} + \frac{1}{v_l}}, 
\end{align}
of the first $|C(j)| - 1$ contestants excluding $i$. 
Next, let $\varepsilon$ be the last contestant that joined contest $j'$, 
we have that
\begin{align} \label{ineq:two}
\frac{p_{j\cup\{\varepsilon\}\setminus\{i\} - 1}}{\sum_{w \in C(j) \setminus \{\varepsilon,i\} }\frac{1}{v_w} + \frac{1}{v_\varepsilon} } \ge \frac{p_{j'}-1}{\sum_{w \in C(j') \setminus \{l,\varepsilon\}} \frac{1}{v_w} + \frac{1}{v_l} + \frac{1}{v_\varepsilon}}
\end{align}
which is $\varepsilon$'s best response condition is to be in contest $j'$ instead of $j$ 
when $i$ has not joined contest $j$.

Given the above two inequalities, we are ready to show that Inequality \ref{ineq:zero} holds. 
For the sake of contradiction, we assume that Inequality \ref{ineq:zero} does not hold. 
We have that 
\begin{align*}
\frac{p_{j'}-1}{\sum_{w \in C(j') \setminus \{l\}} \frac{1}{v_w} + \frac{1}{v_l}} &> \frac{p_{j\cup\{l\}}-1}{\sum_{w \in C(j)\setminus \{i\}} \frac{1}{v_w} + \frac{1}{v_l} + \frac{1}{v_i}} \\
\Longleftrightarrow \frac{\sum_{w \in C(j') \setminus \{l\}} \frac{1}{v_w} + \frac{1}{v_l}}{p_{j'}-1} &< \frac{\sum_{w \in C(j)\setminus \{i\}} \frac{1}{v_w} + \frac{1}{v_l} + \frac{1}{v_i}}{p_{j\cup\{l\}}-1}.
\end{align*}
Applying Inequality \ref{ineq:two}, we have 
\begin{align*}
\frac{\sum_{w \in C(j') \setminus \{l\}} \frac{1}{v_w} + \frac{1}{v_l}}{p_{j'}-1} \ge \frac{\sum_{w \in C(j) \setminus \{\varepsilon,i\} }\frac{1}{v_w} + \frac{1}{v_\varepsilon} }{p_{j\cup\{\varepsilon\}\setminus\{i\}} - 1}. 
\end{align*}
Multiplying the above by  $p_{j\cup\{l\}}-1$, we get 
\begin{align*}
&(p_{j\cup\{l\}}-1)\frac{\sum_{w \in C(j) \setminus \{\varepsilon,i\} }\frac{1}{v_w} + \frac{1}{v_\varepsilon} }{p_{j\cup\{\varepsilon\}\setminus\{i\}} - 1} \\
&\ge \sum_{w \in C(j) \setminus \{\varepsilon,i\} }\frac{1}{v_w} + \frac{1}{v_\varepsilon} + \frac{\sum_{w \in C(j) \setminus \{\varepsilon,i\} }\frac{1}{v_w} + \frac{1}{v_\varepsilon} }{p_{j\cup\{\varepsilon\}\setminus\{i\}} - 1} \\
&\ge \sum_{w \in C(j) \setminus \{\varepsilon,i\} }\frac{1}{v_w} + \frac{1}{v_l} + \frac{1}{v_i}, 
\end{align*}
where the second inequality is due to the fact that $p_{j\cup\{l\}} = p_{j\cup\{\varepsilon\}\setminus\{i\}}$ $+ 1$,
where the third inequality is by first replacing $\frac{1}{v_\varepsilon}$ by $\frac{1}{v_l}$, applying Inequality \ref{ineq:one}, and
noting that $p_{j\cup\{\varepsilon\}\setminus\{i\}} = p_{j\cup\{l\}\setminus\{i\}}$. 
This is a contradiction, and Inequality \ref{ineq:zero} must hold. 
Hence, no contestant $l$ would want to deviate to $j$. 
\end{proof}
Along with the above two claims, we have shown 
that Algorithm 2 finds a PSNE in polynomial time.
\end{proof}

\paragraph{Proof for Theorem 3.}
\emph{In a uniform SSTC game with $n\ge2m$ contestants, every PSNE $\mathbf{a}^*$ achieves optimal OE.} (We also provide additional results for the case of $n<2m$.)
\begin{proof}
We first characterize the overall effort in the PSNE of uniform SSTC.
\begin{lemma}
In any PSNE $\mathbf{a}^* \in A^n$, 
$OE(\mathbf{a}^*) = 0$ when $n \le m$, and 
$OE(\mathbf{a}^*) = v [ (m- (n \bmod m))\left(\frac{\left\lfloor\frac{n}{m}\right\rfloor - 1}{\left\lfloor\frac{n}{m}\right\rfloor}\right)
      + (n \bmod m)\left(\frac{\left\lfloor\frac{n}{m}\right\rfloor}{\left\lfloor\frac{n}{m}\right\rfloor+1}\right)]$
when $n > m$. 
\end{lemma}
\begin{proof}

When $n \le m$, it is optimal for each contestant  
to select a different contest and obtains 
a value $v$, which is the highest utility value. 
The total effort is zero 
as each contest has at most one contestant. 

In the case of $n > m$, we first claim that 
in every PSNE, the 
difference of the numbers of 
contestants between any two contests is no more than one. 
For the sake of contradiction, suppose 
there is a PSNE $\mathbf{a}^* \in A^n$ 
such that there exist contests $j,j', $ 
$c_j(a_i^*,\mathbf{a}_{-i}^*) - c_{j'}(a_i^*,\mathbf{a}_{-i}^*) \ge 2$. 
Consider a contestant $i$ who selects contest $j$ (i.e., $a_{ij} = 1$). 
The utility of $i$ is 
$u_i(a_i^*, a_{-i}^*) = v\frac{1}{c_j(a_i^*,\mathbf{a}_{-i}^*)^2} 
<  v\frac{1}{c_{j'}(a_i',\mathbf{a}_{-i}^*)^2} = u_i(a_i',\mathbf{a}_{-i}^*)$, 
which is larger if $i$ deviates to $j'$ with $a_{ij'}' = 1$.
Thus, in any PSNE, the difference of the numbers of contestants 
between two contests is no more than one. 
As a result, it is not hard to see that 
there are $m -  (n\mod m)$ contests 
that have $\left\lfloor\frac{n}{m}\right\rfloor$ contestants 
while the remaining $(n \mod m)$ contests have
$\left\lfloor\frac{n}{m}\right\rfloor+1$ contestants. 
\end{proof}

Next, we compute the overall effort in optimal solutions.  
\begin{lemma}
In any $\mathbf{a}^{\text{opt}} \in A^n$,
$OE(\mathbf{a}^{\text{opt}}) = v[\frac{(n - 3(n \bmod 2))}{2}\frac{1}{2} + ({n \bmod 2})\frac{2}{3}]$ when $2 \le n \le 2m$
and $OE(\mathbf{a}^{\text{opt}}) = v [ (m- (n \bmod m))\left(\frac{\left\lfloor\frac{n}{m}\right\rfloor - 1}{\left\lfloor\frac{n}{m}\right\rfloor}\right) 
+ (n \bmod m)\left(\frac{\left\lfloor\frac{n}{m}\right\rfloor}{\left\lfloor\frac{n}{m}\right\rfloor+1}\right)]$ 
when $n > 2m$. 
\end{lemma}
\begin{proof}
We consider two cases where
the number of contestants is less than twice of 
the contests ($2 \le n \le 2m$) and 
the number of contestants is more than twice of the contests ($n > 2m$). 

The first thing to note is that the 
overall effort generated from each 
contest has a diminishing return property or concave 
(i.e., $v \frac{ c_j(a_i,\mathbf{a}_{-i}) - 1}{c_j(a_i,\mathbf{a}_{-i}) }$ )
when they are more than one contestant. 
As a result, it is optimal for the platform to place contestants 
into the contests with the least numbers of contestants. 

In the case of $2 \le n \le 2m$, 
it is never optimal for the platform to put 
only one contestant in a contest because 
a contest with only one contestant would yield zero effort.  
Thus, each contest should have at least two contestants. 
In an optimal solution, it is clear that when there is an odd number of contestants, 
one contest should have three contestants 
while $\frac{n - 3(n \bmod 2)}{2}$ contests have two contestants due 
to the diminishing return property. 
Thus, the overall effort is $v[\frac{(n - 3(n \bmod 2))}{2}\frac{1}{2} + ({n \bmod 2})\frac{2}{3}]$.

In the case of $n > 2m$, 
each contest has at least two contestants. 
Moreover, it is optimal to place contestants 
into the contests with the least numbers of contestants. 
Similar to the above lemma, 
there are $m -  (n\mod m)$ contests 
that have $\left\lfloor\frac{n}{m}\right\rfloor$ contestants 
while the remaining $(n \mod m)$ contests have
$\left\lfloor\frac{n}{m}\right\rfloor+1$ contestants. 
\end{proof}

Combining the above lemmas, we derive the ratio $OE(\mathbf{a}^{\text{opt}})/OE(\mathbf{a}^*)=$
\begin{equation*}
\begin{cases*}
      \frac{[\frac{(n - 3(n \bmod 2))}{2}\frac{1}{2} + ({n \bmod 2})\frac{2}{3}]}
       {(m- (n \bmod m))\left(\frac{\left\lfloor\frac{n}{m}\right\rfloor - 1}{\left\lfloor\frac{n}{m}\right\rfloor}\right) + 
       (n \bmod m)\left(\frac{\left\lfloor\frac{n}{m}\right\rfloor}{\left\lfloor\frac{n}{m}\right\rfloor+1}\right)} & $m < n \le 2m$ \\
      1 & $n > 2m$ 
\end{cases*},
\end{equation*}
where the top expression simplifies to 1 when $n=2m$.
\end{proof}

\paragraph{Upper Bound of Individual Effort in a Tullock Contest.}

Consider a Tullock contest with $n$ contestants as formalized in Section 2.1 and let $c_i(e_i)=c_{i}\cdot e_{i}^{d_{i}}$ with $c_i,d_i>0$.
Notice that $i$'s utility, $u_i(e_i, \mathbf{e}_{-i}) = v_i p_i(\mathbf{e}) - c_i(e_i)$, is maximized at $v_{i}-c_{i}(e_{i})$ when $i$ is the sole active contestant.
Furthermore, since the utility is nonnegative (otherwise $i$ would choose $e_{ij}=0$ which yields a utility of 0), it follows that
\begin{align*}
    0 \le v_{i}-c_{i}(e_{i}) &= v_{i}-c_{i}\cdot e_{i}^{d_{i}}\\
    \Longleftrightarrow c_{i}\cdot e_{i}^{d_{i}} &\le v_{i}\\
    \Longleftrightarrow e_{i} &\le \left(\frac{v_{i}}{c_{i}}\right)^{1/d_{i}}.
\end{align*}
Thus, $i$'s maximum effort is $e_{i}^{\max}=(v_{i}/c_{i})^{1/d_{i}}$.

\subsection{Complete Implementation Details}\label{app:imp}

All experiments were conducted under Ubuntu 20.04 on a Linux virtual machine equipped with NVIDIA GeForce RTX 3050 Ti GPU and 12th Gen Intel(R) Core(TM) i7-12700H CPU @2.3GHz.
We chose GPT-4o mini as our pretrained LLM.
The full Python 3.10 implementation accompanies this submission as supplementary material (see Section~5 for our release commitment).

We provide the remaining LLMScore hyperparameters here; the problem settings SSTC~1--4 and our training-on-SSTC~3 protocol are described in Section~5.
We conduct three runs of LLMScore and report the average performance in median OE and median max regret across 1,000 instances.
The LLM temperature is fixed at 1.
The evaluation resource $t$ is defined in terms of the maximum number of populations, for which we set to 20. 
The population size $\eta$ is 20, and the maximum running time for each instance is 180 seconds.
We set the threshold $\tau$ for penalizing $g(h)$ to 0.3.
For the proposed prompt-refinement module (Section~4), we set the size of top individuals in the population, $\lambda$, to 3; the number of maximum consecutive populations without improvement (for detecting stagnation), $t$, to 3; the size of the top individuals produced by a prompt strategy (for computing its score), $\delta$, to 3; and the maximum size of the set of prompt strategies to 5.
For quality control of the LLM-generated prompt strategies, we only accept one if at least one of its produced algorithms is valid, i.e., not experiencing timeout or coding errors.

All the above choices were made mainly to ensure reasonable runtime (typically a few hours) without hampering the generation of high-quality algorithms.

\paragraph{Effort Computation.}\label{par:compute-e}
For SSTC 3 and SSTC 4, we compute $\tilde{e}_{ij}(a_i,\mathbf{a}_{-i})$ for contestants $i\in C(j)$ in each contest $j\in M$ by running best-response dynamics \citep{fudenberg1998theory} detailed in Algorithm \ref{alg:ibr-e}, where we set \texttt{max\_iter} to 100, \texttt{tol} to $10^{-6}$, $\zeta$ to 1, and initialize $\mathbf{e}$ to a vector of all 1's\footnote{Given the generation of SSTC instances described in Section~5, our initialization conforms to the derived $e_{i}^{\max}$ (see ``\ref{app:proofs} Missing Proofs'').}.
We use \href{https://docs.scipy.org/doc/scipy/reference/generated/scipy.optimize.minimize_scalar.html}{Brent's algorithm} \citep{brent2013algorithms} for efficient nonconvex optimization to approximate $\hat{e}_i$ in Step \ref{step:nonconvex-opt}. 

\addtocounter{algorithm}{2}
\begin{algorithm}
    \caption{Best Response Dynamics in Tullock Contests Without Equilibrium Characterization}
    \label{alg:ibr-e}
    \begin{algorithmic}[1]
        \REQUIRE A Tullock contest with $n$ contestants; $\texttt{max\_iter}\in\mathbb{N}^+$; $\texttt{tol}\in\mathbb{R}^+$; damping coefficient $\zeta\in[0,1]$
        \ENSURE Approximate equilibrium effort vector $\tilde{\mathbf{e}}\in\mathbb{R}_0^{+\,n}$
        \STATE Initialize $\mathbf{e}=(e_1,\ldots,e_n)$
        \FOR{$l = 1,\ldots,\texttt{max\_iter}$}
        \STATE Let $\mathbf{e}'=\mathbf{e}$
        \FOR{$i = 1, \ldots, n$}
        \STATE\label{step:nonconvex-opt} Compute $\hat{e}_i=\argmax_{e_i\in[0,e_i^{\max}]}u_i(e_i,\mathbf{e}_{-i})$ via nonconvex optimization \hfill \doc{\# $u_i(e_i,\mathbf{e}_{-i})$ previously defined in Section 2.1}
        \STATE Update $e_i=\zeta\cdot \hat{e}_i+(1-\zeta)\cdot e_i$ 
        \ENDFOR
        \IF{$\lVert \mathbf{e}'-\mathbf{e}\rVert<\texttt{tol}$}
        \STATE Stop program and return $\mathbf{e}$
        \ENDIF
        \ENDFOR
    \end{algorithmic}
\end{algorithm}

\paragraph{Iterated Best Response.}
Algorithm \ref{alg:ibr} details IBR.
Due to its dependence on the initial (randomly generated) self-selection outcome $\mathbf{a}$, for each instance, we run 10 trials of IBR with different starting $\mathbf{a}$ and report the average performance (OE and max regret) across runs.
If cycle occurs, i.e., $\mathbf{a}$ at iteration $l$ exactly matches $\mathbf{a}$ at any prior iteration $0<l'<l-1$, we stop the program early and average the performance of the actions from $(\mathbf{a}^{l'},\ldots,\mathbf{a}^{l-1})$.

\begin{algorithm}
    \caption{Iterated Best Response}
    \label{alg:ibr}
    \begin{algorithmic}[1]
        \REQUIRE An SSTC game with $n$ contestants; $\texttt{max\_iter}\in\mathbb{N}^+$
        \ENSURE A self-selection outcome $\mathbf{a}$ (or a list thereof when cycle occurs)
        \STATE Initialize $\mathbf{a}\in A^n$ randomly satisfying $k_i$ for every $i\in N$
        \FOR{$l = 1,\ldots,\texttt{max\_iter}$}
        \FOR{$i = 1, \ldots, n$}
        \STATE Let $a'_i=\mathbf{0}$
        \STATE Let $\mathcal{BR}_i(\mathbf{a}_{-i}) = \argmax_{a_i \in A_i} u_i(a_i, \mathbf{a}_{-i})$ \hfill \doc{\# Recall that $\sum_{j\in M}a_{ij}=k_i$}
        \STATE $a'_{ij}=1$ for $j\in M'$, where $M'\in \mathcal{BR}_i(\mathbf{a}_{-i})$ \hfill \doc{\# $|M'|=k_i$}
        \STATE Update $a_i=a'_i$
        \ENDFOR
        \IF{cycle detected for $\mathbf{a}$}
        \STATE Stop program
        \ENDIF
        \ENDFOR
    \end{algorithmic}
\end{algorithm}

\end{document}